\documentclass{article}

\usepackage[preprint]{neurips_2026}

\usepackage{amsmath,amsfonts,bm}

\def\eqref#1{(\ref{#1})}

\def\1{\bm{1}}

\DeclareMathAlphabet{\mathsfit}{\encodingdefault}{\sfdefault}{m}{sl}
\SetMathAlphabet{\mathsfit}{bold}{\encodingdefault}{\sfdefault}{bx}{n}

\usepackage[utf8]{inputenc} 
\usepackage[T1]{fontenc}    
\usepackage{hyperref}
\usepackage{url}
\usepackage{microtype}
\usepackage[final]{graphicx}
\usepackage{wrapfig}
\usepackage{subcaption}
\usepackage{booktabs}
\usepackage{enumitem}
\usepackage{multirow}
\usepackage{bm}
\usepackage{algorithm}
\usepackage{algorithmic}
\usepackage[dvipsnames]{xcolor}
\definecolor{shsunder}{named}{BrickRed}
\definecolor{shsover}{named}{RoyalBlue}
\usepackage{nicefrac}
\usepackage{makecell}
\usepackage{adjustbox}
\newcommand\blfootnote[1]{%
  \begingroup
  \renewcommand\thefootnote{}\footnotetext{#1}%
  \endgroup
}
\usepackage{amsmath,amssymb,amsthm,amsfonts}
\usepackage{mathtools}

\theoremstyle{plain}
\newtheorem{theorem}{Theorem}
\newtheorem{proposition}[theorem]{Proposition}

\theoremstyle{remark}

\theoremstyle{remark}
\newtheorem{assumption}[theorem]{Assumption}

\newcommand{\halfimg}[1]{%
  \adjincludegraphics[width=\linewidth, trim={0 0 {0.5\width} 0}, clip]{#1}%
}

\title{Systematic Hazard Sampling: Minimal-Variance Inference for Discrete Diffusion and Flow Models}

\author{%
  Seunghwan~Jang$^{*\dagger}$\\
  Nanyang Technological University\\
  Singapore\\
  \texttt{seunghwa001@e.ntu.edu.sg}
  \And
  Wonje~Jeung$^{\ddagger}$\\
  University of Michigan\\
  Ann Arbor, MI, USA\\
  \texttt{jeung@umich.edu}
  \And
  SooJean~Han\\
  KAIST\\
  Daejeon, Republic of Korea\\
  \texttt{soojean@kaist.ac.kr}
}

\begin{document}

\maketitle

\blfootnote{$^{*}$Corresponding author.}
\blfootnote{$^{\dagger}$Work done at KAIST. $^{\ddagger}$Work done at Yonsei University.}

\begin{abstract}
Uniform-noise discrete diffusion and flow models generate sequences non-autoregressively through iterative, context-dependent token replacements.
However, these models are typically formulated as time-inhomogeneous continuous- or discrete-time Markov chains (CTMC/DTMC), sampled using independent Bernoulli change decisions per discretization step.
This induces Poisson-binomial variance in per-position jump counts that grows with the number of required edits, leading to the common under-editing (residual noise) and over-editing (cascading substitutions) failure modes that degrade sample quality.
We identify this sampler-induced variance as an orthogonal source of degradation, distinct from model-side errors and addressable purely at inference time.
We propose \emph{Systematic Hazard Sampling} (SHS), a training-free, drop-in, and hyperparameter-free inference principle for any sampler that admits a stay-vs.-replace decomposition.
SHS models per-token edits as events driven by cumulative hazard (CTMC) or jump mass (DTMC) and triggers an edit whenever this quantity exceeds unit-spaced thresholds with a single random phase per position.
For any fixed cumulative mass, this preserves the expected jump count while achieving the minimum conditional variance possible among unbiased integer estimators (at most \(1/4\)), without altering per-jump destination sampling.
Experiments on four uniform-noise discrete diffusion and flow language models spanning $\sim$110M to $\sim$3B parameters show that SHS consistently improves sample quality across numbers of function evaluations (NFE).
Code is available at \url{https://github.com/Jang-seunghwan/Systematic-Hazard-Sampling}.
\end{abstract}

\section{Introduction}
\label{introduction}

Non-autoregressive generation for high-dimensional discrete data (e.g., text, code) is a promising approach that can significantly reduce inference latency by updating all tokens in parallel.
A growing family of methods recasts discrete generation as a time-inhomogeneous Markov process (CTMC or DTMC) that transports an easy prior \(p_0\) to the data distribution \(p_1\), including discrete diffusion models (D3PM~\citep{austin2021d3pm}, SEDD~\citep{lou2024sedd}, UDLM~\citep{schiff2025simple}) and discrete flow matching (DFM)~\citep{gat2024dfm}.
However, in uniform-noise-start models, meaningful generation typically requires multiple context-dependent edits per position.

This multi-jump flexibility comes at a practical cost in the standard step-based sampler.
At each discretization step, every position independently decides whether to jump via a Bernoulli (or categorical) draw.
The total number of edits per position is therefore a Poisson-binomial random variable whose variance \(\sum_k p_{ik}(1 - p_{ik})\) can grow linearly with the cumulative jump mass.
This sampler-induced variance produces two characteristic failure modes at the tails of the jump-count distribution.
\textit{Under-editing} (too few jumps) occurs when insufficient substitutions leave residual noise or local inconsistencies, preventing full integration of contextual information before the process terminates.
Conversely, \textit{over-editing} (too many jumps) occurs when excessive substitutions cascade, disrupting even coherent segments and producing repetitions or distortions.
Crucially, refining the time grid does not eliminate this variance: as the per-step Bernoulli probabilities shrink, $\sum_k p_{ik}(1-p_{ik})$ approaches its Poisson limit $S$ (the cumulative jump mass), so the variance \emph{grows} with the number of function evaluations (NFE) rather than vanishing in the small-step limit.


\begin{figure}
    \centering
    \includegraphics[width=1.0\linewidth]{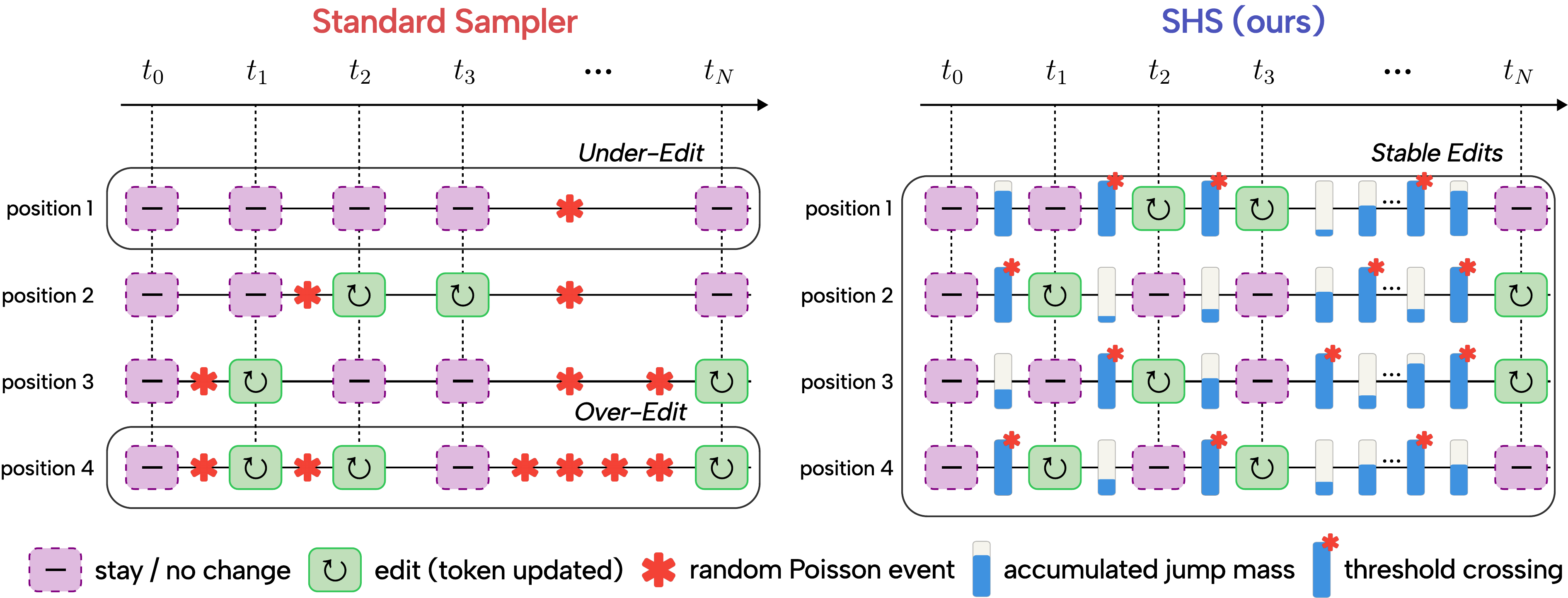}
    \caption{Standard step-based sampling (left) makes independent per-step Bernoulli edit decisions, leading to \textit{under-edit} (position 1) and \textit{over-edit} (position 4) failure modes. SHS (right) accumulates jump mass across steps and triggers edits at unit-spaced thresholds with a single random phase per position, producing stable, low-variance edit trajectories.}
    \label{fig:overview}
    \vspace{-1em}
\end{figure}

These failure modes are not inherent to the uniform-noise formulation itself but are artifacts of the independent per-step sampling mechanism.
Absorbing-state (mask-start) methods \citep{austin2021d3pm, campbell2022ctdd}---where each position unmasks exactly once---yield simple, low-variance trajectories, but reduce generation to any-order parallel unmasking without iterative refinement.
Uniform-noise initialization \citep{austin2021d3pm, lou2024sedd, schiff2025simple} allows each position to undergo multiple context-dependent jumps, providing genuine self-correction capability at the cost of increased trajectory complexity.
Because existing step-based samplers for uniform-noise models typically realize replacements through independent per-step jumps, the resulting jump count and timing can be vulnerable to this sampler-induced variance, or ``sampler luck''.

The goal of this paper is to preserve the expressive multi-jump dynamics of uniform-noise-start discrete generative models while suppressing this unnecessary sampler variance.
To this end, we propose \textit{Systematic Hazard Sampling (SHS)}, which requires no retraining, no additional hyperparameters, and no architectural modification---only a change to the inference-time sampling rule.
SHS uses the cumulative hazard 
\(S(t) = \int_0^t \lambda(s) \, ds\) in CTMC samplers (or the corresponding cumulative jump mass in DTMC samplers) to systematically place events in this space.
SHS triggers jumps when the cumulative hazard \(S(t)\) crosses integer boundaries offset by a random \(\theta \sim \mathrm{Uniform}(0,1)\).
Unlike prior works,
which use probabilistic coin flips at fixed time steps to trigger the jumps, SHS preserves the expected jump count for any fixed cumulative hazard/jump mass while bounding the jump count variance to a theoretical minimum (at most $1/4$, \emph{uniformly in NFE}).
Moreover, unlike top-$k$/nucleus filtering, which biases the model's learned destination distribution at inference time, SHS leaves the per-jump destination kernel $q_{i,k}$ unchanged and modifies only the temporal coupling of the stay-vs.-replace events.
Consequently, SHS mitigates the long-tail waiting times and ``sampler luck'' that are especially pronounced in uniform-noise starts; this simultaneously suppresses under- and over-editing.

Empirically, we evaluate SHS on four uniform-noise discrete diffusion and flow models spanning $\sim$110M to $\sim$3B parameters (UDLM, GIDD, DFM, and a 3B uniform diffusion model).
SHS consistently improves generative perplexity (Gen.\ PPL) across NFE budgets; the predicted jump-count concentration holds in the real sampler, and under-editing turns out to be the costly side of count deviation (Sections~\ref{sec:jump_variance}--\ref{sec:count_deviation_quality}).

\section{Preliminaries}
\label{sec:prelim_related}
\label{prelim}

\label{sec:dlm_scope}
Let $\mathcal{V}$ be a vocabulary, $N$ the sequence length, $n$ the number of discretization steps, and $x=(x_1,\ldots,x_N)\in\mathcal{V}^N$ a sequence. Discrete diffusion and flow language models define time-inhomogeneous stochastic processes that transport a simple source distribution $p_0$ to the data distribution $p_1$. The source and transition family determine whether generation is essentially single-edit or multi-edit at each position.

\paragraph{Masked/absorbing vs. uniform-noise starts.}
Absorbing-state or mask-start models augment the vocabulary with a special state $\mathtt{[MASK]}$ and generate by resolving masked positions into vocabulary tokens. In the standard unmask-only sampler, once a position leaves
$\mathtt{[MASK]}$, it is fixed:
\begin{equation}
p_{ik}(x)=0
\qquad
\text{whenever } x_i\in\mathcal{V},
\label{eq:absorbing_unmask_only}
\end{equation}

where $p_{ik}(x)$ is the probability that position $i$ changes at step $k$.
Thus each position can edit at most once: with $J_i$ denoting the total jump count at position $i$ over the trajectory, $J_i \in \{0,1\}$ almost surely (see Appendix~\ref{app:maskstart_single_edit} for a formal statement).
This monotone structure is simple and stable, and has led to strong masked/absorbing diffusion language models (DLMs) \citep{austin2021d3pm,sahoo2024mdlm,shi2024md4,ou2025radd}. 
Recent variants of mask-start models restore revision ability by adding remasking or correction modules \citep{wang2025remdm,schiff2026proseco}. In contrast, uniform-noise models \citep{schiff2025simple,rutte2025generalized} initialize positions from vocabulary-valued noise and allow tokens to be replaced repeatedly, giving the sampler a genuine multi-edit trajectory; recent scaling studies of uniform and interpolating diffusion at billion-parameter scale \citep{vonrutte2025scalingbehaviordiscretediffusion,sahoo2026scalingbeyond} make this regime practically important. Discrete flow matching provides a related probability-path view for categorical generation \citep{gat2024dfm,monsefi2025fs}. All of these directions improve the trained model itself---its noise structure, training objective, or learned distribution---but leave an orthogonal source of degradation, sampler-induced variance in step-based simulation, largely unaddressed.

\paragraph{A stay-vs.-replace view of CTMC/DTMC samplers.}
Discrete generative models differ in how they define or learn the reverse dynamics: structured DTMC kernels \citep{hoogeboom2021argmax,austin2021d3pm}, CTMC reverse processes \citep{campbell2022ctdd}, score or ratio parameterizations \citep{lou2024sedd}, uniform/hybrid marginal-transition constructions \citep{schiff2025simple,rutte2025generalized}, and discrete flow-matching probability paths \citep{gat2024dfm,monsefi2025fs}. We do not require these formulations to be identical. We only require that the \emph{sampler} admit a local stay-vs.-replace decomposition.

\textit{CTMC form.}
A time-inhomogeneous CTMC on $\mathcal{V}^N$ is specified by a generator $Q_t$.
For each position $i$, write $\lambda_i(t,x)\ge0$ for the escape rate and $q_{i,t}(\cdot\mid x)\in\Delta(\mathcal{V}\setminus\{x_i\})$ for the destination distribution (probability simplex) conditional on an edit.
With $x^{(i\leftarrow v)}$ denoting the sequence obtained by replacing $x_i$ with $v$, the local off-diagonal rates can be written as

\begin{equation}
Q_t\!\left(x^{(i\leftarrow v)}\mid x\right)
=
\lambda_i(t,x)\,q_{i,t}(v\mid x),
\qquad v\neq x_i,
\label{eq:local_rates}
\end{equation}

and $Q_t(x\mid x)=-\sum_i\lambda_i(t,x)$, the negative total escape rate, so that each row of the generator sums to zero \citep{norris1997markov}. 
In an Euler or $\tau$-leaping discretization on grid points $t_k = kh$ with step size $h = 1/n$, the model supplies a per-step change mass $p_{ik}\in[0,1]$, for example $p_{ik}=1-\exp(-h\lambda_i(t_k,x))$ under a frozen-rate Poisson clock, or $p_{ik}\approx h\lambda_i(t_k,x)$ for small steps.

\textit{DTMC form.}
For a discrete-time reverse kernel, let $P_{i,k}(\cdot\mid x)\in \Delta(\mathcal{V})$ be the per-position categorical update distribution. 
Any such kernel can be decomposed as
\begin{equation}
p_{ik}(x)=1-P_{i,k}(x_i\mid x),
\qquad
q_{i,k}(v\mid x)=\frac{P_{i,k}(v\mid x)}{p_{ik}(x)}
\quad (v\neq x_i),
\label{eq:dtmc_decomp}
\end{equation}
where $p_{ik}(x)$ is the probability that the position changes and $q_{i,k}$ is the destination distribution conditional on changing. If $p_{ik}(x)=0$, the destination distribution is irrelevant.

\textit{Unified event view.}
Both CTMC $\tau$-leaping and DTMC categorical updates can be written using a per-step edit indicator $B_{ik}\in\{0,1\}$ and a replacement token $V_{ik}\in\mathcal{V}$ as
\begin{equation}
B_{ik}\sim\mathrm{Bernoulli}(p_{ik}),
\qquad
X^i_{t_{k+1}}=
\begin{cases}
X^i_{t_k}, & B_{ik}=0,\\
V_{ik},\quad V_{ik}\sim q_{i,k}(\cdot\mid X_k), & B_{ik}=1.
\end{cases}
\label{eq:standard_event_sampler}
\end{equation}
Define the cumulative hazard or jump mass
\begin{equation}
S_{i,k}=\sum_{j=0}^{k-1}p_{ij},
\qquad
S_{i,n}=\sum_{j=0}^{n-1}p_{ij}
\approx \int_0^1 \lambda_i(t,X_t)\,dt.
\label{eq:cum_jump_mass}
\end{equation}

The variance statements below are conditional on a prescribed or realized sequence of masses $\{p_{ik}\}_{k=0}^{n-1}$; in the full sampler, these masses are state-dependent.

\paragraph{Jump-count variance under grid refinement.}
Conditional on a fixed mass sequence $\{p_{ik}\}_{k=0}^{n-1}$, the standard sampler draws independent Bernoulli indicators. The per-position jump count
\[
J_i=\sum_{k=0}^{n-1}B_{ik}
\]
therefore follows a Poisson-binomial law with
\begin{equation}
\mathbb{E}[J_i\mid\{p_{ik}\}]=S_{i,n},
\qquad
\mathrm{Var}(J_i\mid\{p_{ik}\})
=
\sum_{k=0}^{n-1}p_{ik}(1-p_{ik}).
\label{eq:poisson_binomial_variance}
\end{equation}
Refining the time grid does not remove this conditional jump-count variance.
Fixing a single position and dropping the position index $i$ for clarity (the index returns in Section~\ref{method}), for a fixed cumulative mass $S=\sum_{k=1}^{n}p_k$,
\begin{equation}
\mathrm{Var}(J\mid\{p_k\})
=
\sum_{k=1}^{n}p_k(1-p_k)
=
S-\sum_{k=1}^{n}p_k^2.
\label{eq:fixed_mass_variance}
\end{equation}
If $\max_k p_k\to0$ while $\sum_k p_k=S$, then
\begin{equation}
\sum_{k=1}^{n}p_k^2
\le
\Bigl(\max_k p_k\Bigr)\sum_{k=1}^{n}p_k
=
\Bigl(\max_k p_k\Bigr)S
\to0,
\qquad
\mathrm{Var}(J\mid\{p_k\})\to S.
\label{eq:poisson_limit_variance}
\end{equation}
Thus finer discretization may reduce numerical integration error, but it does not eliminate the independent Bernoulli jump-count fluctuations; in the small-step limit, the count variance approaches the Poisson limit associated with cumulative hazard $S$.\footnote{The common uniform refinement $p_k=S/n$ satisfies the condition exactly. A frozen-rate CTMC step with $p_k=1-\exp(-S/n)$ also satisfies $\max_kp_k\to0$ and $\sum_kp_k\to S$, yielding the same limit.}

\section{Systematic Hazard Sampling}
\label{method}

We propose \textit{Systematic Hazard Sampling (SHS)}, adapting systematic event simulation to the \emph{step-based} samplers used in CTMC/DTMC discrete generative models.
SHS replaces the independent per-step change decisions $B_{ik}{\,\sim\,}\mathrm{Bernoulli}(p_{ik})$ (Section~\ref{prelim}) with \emph{single-phase} systematic placement in cumulative hazard/jump-mass space.
Concretely, SHS maintains the cumulative mass $S_{i,k}$ from \eqref{eq:cum_jump_mass} and triggers a jump whenever $S_{i,k}$ crosses unit-spaced thresholds $\{\theta_i + m\}_{m\in\mathbb{Z}_{\ge 0}}$ for a single random phase $\theta_i\sim\mathrm{Uniform}(0,1)$.
The complete pseudocode of SHS is shown in Algorithm~\ref{alg:shs}.

\paragraph{One boundary per step.}
Since $p_{ik}\le 1$ holds by definition for DTMC kernels and is a prerequisite for valid CTMC $\tau$-leaping, at most one unit boundary can be crossed per step, so the if statement on Line 14 is sufficient.
Otherwise, Line 14 can be replaced by a while-loop that triggers all crossed boundaries.

SHS changes only the \emph{timing} of jump events, by coupling the Bernoulli decisions across steps through the single phase $\theta_i$; Appendix~\ref{sec:hazard_strat} quantifies the resulting shift in edit timing.
The destination sampling $q_{i,k}$ is unchanged: whenever an edit is triggered, the replacement token is drawn from the same destination kernel as in the standard sampler.

\paragraph{Jump-count concentration.}
Let $S_i^{\mathrm{tot}}:=S_{i,n}=\sum_{k=0}^{n-1}p_{ik}$ denote the total cumulative hazard/jump mass accumulated at position $i$.
Write $S_i^{\mathrm{tot}} = I_i + f_i$ with $I_i=\lfloor S_i^{\mathrm{tot}}\rfloor$ and $f_i\in[0,1)$.
Conditioning on a fixed realized $S_i^{\mathrm{tot}}$, SHS counts how many thresholds $\theta_i+m$ are crossed by $S_i^{\mathrm{tot}}$, so we have
\begin{equation}
J_i \;=\; I_i + \mathbf{1}[\theta_i < f_i],
\label{eq:shs_rounding}
\end{equation}
so $J_i\in\{I_i, I_i+1\}$ \emph{conditional on the realized cumulative mass} $S_i^{\mathrm{tot}}$, and the conditional long-tail jump-count fluctuations of step-based Bernoulli sampling are eliminated for that fixed cumulative mass $S_i^{\mathrm{tot}}$. 
In the full model the masses $p_{ik}(x)$ are state-dependent, so $S_i^{\mathrm{tot}}$ itself is a random variable whose distribution may differ between SHS and standard sampling (see Appendix~\ref{sec:scope} for discussion).
As a result, SHS suppresses the characteristic under-edit (too few substitutions) and over-edit (cascading substitutions) failure modes that arise under uniform-noise initialization.

\begin{algorithm}[h]                         
  \caption{Systematic Hazard Sampling (SHS) for CTMC/DTMC discrete generative models}
  \label{alg:shs}
  {\small
  \begin{algorithmic}[1]
    \STATE \textbf{Input:}
    \STATE \quad - Initial state $ X \sim p_0 $, Time grid $ t_k = kh $ for $ k=0,\dots,n $
    \STATE \quad - A routine $ \textsc{ModelStep}(t_k, X_{t_k}, i) $ that returns: \quad (cf.\ \eqref{eq:local_rates}, \eqref{eq:dtmc_decomp})
    \STATE \quad \quad - change mass $ p_{ik} \in [0,1] $
    \STATE \quad \quad - destination $ q_{i,k}(\cdot \mid X_{t_k}) $
    \STATE \textbf{Output:} Final sample $ X $ at $ t_n $ (i.e., $ t=1 $)
    \FOR{$i=1$ \textbf{to} $N$}
      \STATE $S_i \gets 0$;\quad $m_i \gets 0$;\quad \textbf{Draw} $\theta_i \sim \mathrm{Uniform}(0,1)$
    \ENDFOR
    \FOR{$k = 0,1,\dots,n-1$}
      \FOR{$i=1$ \textbf{to} $N$}
        \STATE $(p_{ik},\,q_{i,k}) \gets \textsc{ModelStep}(t_k, X_{t_k}, i)$
        \STATE $S_i \gets S_i + p_{ik}$
        \IF[/*crossed the next hazard boundary*/]{$S_i \ge \theta_i + m_i$}
          \STATE $X_{t_{k+1}}^i \sim q_{i,k}(\cdot\mid X_{t_k})$;\quad $m_i \gets m_i + 1$
        \ELSE
          \STATE $X_{t_{k+1}}^i \gets X_{t_k}^i$
        \ENDIF
      \ENDFOR
    \ENDFOR
    \STATE \textbf{return} $X_{t_n}$
  \end{algorithmic}
  }
\end{algorithm}

\subsection{Theoretical Properties} \label{subsec:theory}
SHS is built on a simple primitive: \emph{randomized rounding} of a nonnegative real mass $S$ into an integer count $J\in\mathbb{Z}_{\ge 0}$ using a \emph{single} uniform random variable. 
Writing $S = I + f$ with $I := \lfloor S\rfloor$ and $f := S - I \in [0,1)$ (a decomposition we use throughout), the rounding rule is $J = I + \mathbf{1}[\theta < f]$ for $\theta \sim \mathrm{Uniform}(0,1)$.
In SHS, this mass is the total cumulative hazard/jump mass $S_i^{\mathrm{tot}}$ at each position \eqref{eq:shs_rounding}.

\begin{proposition}[\textbf{Jump-Count Unbiasedness}]
\label{prop:unbiased_main}
Fix $S \ge 0$ and draw $\theta \sim \mathrm{Uniform}(0,1)$.
Write $S = I + f$ with $I := \lfloor S \rfloor$ and $f := S - I \in [0,1)$.
Define $J := I + \mathbf{1}[\theta < f]$.
Then $\mathbb{E}[J] = S$.
\end{proposition}

The randomized rounding rule used by SHS produces an integer jump count whose expectation matches the cumulative mass exactly: for a given cumulative mass, SHS does not systematically over- or under-edit relative to the model's intended total.

\begin{proposition}[\textbf{Minimal Variance}]
\label{prop:variance_main}
With $J$ as in Proposition~\ref{prop:unbiased_main},
we have
\begin{equation}
\mathrm{Var}(J) = f(1-f) \le \frac{1}{4}.
\label{eq:min_var_rounding}
\end{equation}
Moreover, among all $\mathbb{Z}_{\ge 0}$-valued
unbiased estimators of $S$,
this variance is the minimum possible.
\end{proposition}

The realized jump count never deviates from the cumulative mass by more than one, and its variance is at most $1/4$ regardless of $S$ or the discretization budget---the smallest variance any integer-valued unbiased rounding can achieve.

\paragraph{Proof sketches.}
Write $S = I + f$ with $I = \lfloor S \rfloor$ and $f = S - I \in [0,1)$.
Then $J = I + B$ where $B = \mathbf{1}[\theta < f] \sim \mathrm{Bernoulli}(f)$.
Hence $\mathbb{E}[J] = I + f = S$ and $\mathrm{Var}(J) = f(1-f) \le 1/4$.
Minimality holds because, for any integer-valued $\widetilde J$ with $\mathbb{E}[\widetilde J]=S$ and $U:=\widetilde J-I$, the identity $(U-f)^2=U(U-1)+(1-2f)U+f^2$ gives $\mathrm{Var}(\widetilde J)=\mathbb{E}[U(U-1)]+f(1-f)\ge f(1-f)$, with equality if and only if $\widetilde J\in\{I,I+1\}$ almost surely.
See Appendix~\ref{appen:variance_proof} for the full proof (and Appendix~\ref{appen:unbiased_proof} for Proposition~\ref{prop:unbiased_main}).

\paragraph{Interpretation.}
For a fixed mass $S$, standard step-based simulation yields a Poisson-binomial count with variance $\sum_k p_k(1-p_k)$, which can scale linearly with $S$ in the many-edit regime.
In contrast, SHS concentrates the jump count to the two adjacent integers ($\lfloor S\rfloor$ or $\lceil S\rceil$), eliminating long-tail ``sampler luck'' in the number of edits.  
This statement is about the binary jump-count process; the categorical destination kernel used at each realized edit is left unchanged.
Section~\ref{sec:jump_variance} measures this concentration in the real sampler, and Section~\ref{sec:count_deviation_quality} shows that the deviations it removes are the ones that cost quality.

\subsection{Reducing Event Variance without Collapsing Destination Diversity}
\label{sec:multimodality}

A natural concern with Proposition~\ref{prop:variance_main} is whether reducing jump-count variance also reduces the multimodality that stochastic sampling provides. 
We separate three sources of randomness in a discrete language model sampler: (i) the initial source sample $X_0\sim p_0$, which acts as the global generation seed; (ii) the binary edit schedule $B=(B_{ik})_{i\le N,\,0\le k<n}\in\{0,1\}^{N\times n}$, which decides when each position is edited; and (iii) the categorical destination draws $V_{ik}\sim q_{i,k}(\cdot\mid X_{t_k})$, which determine what token is written when an edit occurs.
SHS modifies only the second source: it changes the temporal coupling of the binary edit indicators while preserving the initial source distribution $p_0$ and the categorical destination kernels $q_{i,k}$.
The following two structural facts make this asymmetry precise; proofs are in Appendix~\ref{appen:multimodality}.

Write $B^{\mathrm{std}}_{ik}$, $J^{\mathrm{std}}_i$ for the standard sampler and $B^{\mathrm{SHS}}_{ik}$, $J^{\mathrm{SHS}}_i$ for Algorithm~\ref{alg:shs}.
\begin{proposition}[Per-step event-marginal preservation]
\label{prop:event_marginal}
For any fixed mass sequence $\{p_{ik}\}_{k=0}^{n-1}$ with $p_{ik}\le 1$,
\begin{subequations}
\begin{gather}
\mathbb{P}_{\theta_i}\left(B^{\mathrm{SHS}}_{ik}=1 \,\middle|\, \{p_{ij}\}_{j=0}^{n-1}\right)
\;=\; p_{ik}
\;=\; \mathbb{P}\!\left(B^{\mathrm{std}}_{ik}=1 \,\middle|\, \{p_{ij}\}_{j=0}^{n-1}\right).
\label{eq:event_marginal_preservation}\\
\mathbb{E}_{\theta_i}\!\left[J_i^{\mathrm{SHS}} \,\middle|\, \{p_{ij}\}_{j=0}^{n-1}\right]
=\sum_{k=0}^{n-1} p_{ik}
=\mathbb{E}\!\left[J_i^{\mathrm{std}} \,\middle|\, \{p_{ij}\}_{j=0}^{n-1}\right].
\label{eq:terminal_count_expectation_preservation}
\end{gather}
\end{subequations}
where $J_i^s$ is the total number of jump events at position $i$ under sampler $s\in\{\mathrm{SHS},\mathrm{std}\}$.
\end{proposition}

Proposition~\ref{prop:event_marginal} states that at every step, the marginal probability of an edit under SHS equals that under the standard sampler of Section~\ref{sec:prelim_related}.
SHS therefore preserves the model's per-step edit budget along any fixed mass sequence and changes only the coupling between steps; Section~\ref{sec:jump_variance} checks the same conditions on the masses the state-dependent sampler actually produces.

Now, given the full binary event pattern $B=(B_{ik})$, let $\pi_s$ denote the induced law of $B$ for sampler $s\in\{\mathrm{std},\mathrm{SHS}\}$.
Given an externally specified realization $b$ of $B$, let $\mathcal{K}_b$ denote the trajectory law obtained by drawing $X_0\sim p_0$ and applying the common destination rule $q_{i,k}(\cdot\mid X_{t_k})$ at exactly the events specified by $b$.

\begin{proposition}[Destination-kernel identity under fixed events]
\label{prop:destination_kernel_identity}
For any state $x$, position $i$, step $k$, and token $v\neq x_i$,
\begin{equation}
\mathbb{P}^s\!\left(X^i_{t_{k+1}}=v \,\middle|\, B_{ik}=1,\; X_{t_k}=x\right) \;=\; q_{i,k}(v\mid x)
\quad\text{for both } s\in\{\mathrm{std},\mathrm{SHS}\}.
\label{eq:destination_kernel_identity}
\end{equation}
Consequently, both samplers induce the same trajectory kernel $\mathcal{K}_b$ for any fixed event pattern $b$.
\end{proposition}

Once an edit is triggered, both samplers draw the replacement token from the same destination distribution.
SHS does not truncate, sharpen, or otherwise reshape what the model wants to write at any edit; it only changes when edits are scheduled.

\paragraph{Variance decomposition.}
Defining $\mu_F(b):=\mathbb{E}_{\mathcal{K}_b}[F(X_n)]$ and $\sigma_F^2(b):=\mathrm{Var}_{\mathcal{K}_b}[F(X_n)]$ for any statistic $F(X_n)$, the law of total variance gives
\begin{equation}
\mathrm{Var}_{B\sim\pi_s,\;X\sim\mathcal{K}_B}[F(X_n)]
\;=\;
\underbrace{\mathbb{E}_{B\sim\pi_s}[\sigma_F^2(B)]}_{\text{categorical variability at fixed events}}
\;+\;
\underbrace{\mathrm{Var}_{B\sim\pi_s}[\mu_F(B)]}_{\text{event-pattern variability}}.
\label{eq:totvar_decomp}
\end{equation}
By Proposition~\ref{prop:destination_kernel_identity}, $\mu_F$ and $\sigma_F^2$ are sampler-independent functions; only $\pi_s$ depends on the sampler.
Proposition~\ref{prop:variance_main} controls the count-driven part of $\pi_s$. Conditional on $S_i^{\mathrm{tot}}$, SHS forces $J_i{\,\in\,}\{\lfloor S_i^{\mathrm{tot}}\rfloor,\lceil S_i^{\mathrm{tot}}\rceil\}$ with $\mathrm{Var}(J_i\mid S_i^{\mathrm{tot}}){\,\le\,}1/4$, whereas the standard sampler has $\mathrm{Var}(J_i \mid\{p_{ik}\}_{k=0}^{n-1})=\sum_k p_{ik}(1-p_{ik})$, approaching the Poisson limit $S_i^{\mathrm{tot}}$ as the grid is refined. 
Under some $L_C$-Lipschitz sensitivity condition on the count-level conditional mean, this bound propagates to the count-driven component of $\mathrm{Var}_B[\mu_F(B)]$ as $L_C^2 \sum_i \mathrm{Var}(J_i\mid S_i^{\mathrm{tot}}) \le L_C^2 N/4$ (see Appendix~\ref{appen:lipschitz} for the precise statement and derivation).
The conditional categorical-variability function $\sigma_F^2(b)$ depends only on the destination kernels $q_{i,k}$, which SHS does not modify.
Thus, in this fixed-event view, the only sampler-dependent object is $\pi_s$.

\paragraph{Empirical predictions.}
Two consequences are tested in Section~\ref{sec:experiments}. Jump-count variance separates as predicted, with SHS capped at $1/4$ uniformly while the standard sampler grows toward its Poisson limit
(Section~\ref{sec:jump_variance}). And since only the event-pattern term changes, whatever diversity SHS removes should be path variance rather than destination variance, which Section~\ref{sec:exp-text} confirms against real text and against the generating model's own likelihood.

\begin{table}
    \centering
    \caption{\textbf{Generative perplexity under NFE budgets} (Gen.\ PPL, lower is better; bold = better between Standard and SHS). $\Delta$ shows SHS improvement over Standard.
    $\pm$: standard deviation over seeds (5 for UDLM and GIDD, 3 for DFM; single seed for 3B).} \label{tab:main_genppl}
    \resizebox{\columnwidth}{!}{%
    \begin{tabular}{c cc cc c| cc cc c}
      \toprule
      & \multicolumn{5}{c}{UDLM} & \multicolumn{5}{c}{GIDD} \\ \midrule \multirow{2}{*}{NFE} & \multicolumn{2}{c}{Standard} & \multicolumn{2}{c}{SHS} & \multirow{2}{*}{$\Delta$} & \multicolumn{2}{c}{Standard} & \multicolumn{2}{c}{SHS} & \multirow{2}{*}{$\Delta$} \\
      \cmidrule(lr){2-3} \cmidrule(lr){4-5} \cmidrule(lr){7-8} \cmidrule(lr){9-10}
      & Entropy & Gen.\ PPL$\downarrow$ & Entropy & Gen.\ PPL$\downarrow$ & & Entropy & Gen.\ PPL$\downarrow$ & Entropy & Gen.\ PPL$\downarrow$ &\\
      \midrule
         4 & $6.633 \pm 0.008$ & $559.3 \pm 6.3$ & $\bm{6.586 \pm 0.007}$ & $\bm{535.6 \pm 5.6}$ & $-4.2\%$ & $7.479 \pm 0.010$ & $1696.2 \pm 14.5$ & $\bm{7.438 \pm 0.010}$ & $\bm{1606.4 \pm 13.6}$ & $-5.3\%$\\
         8 & $6.739 \pm 0.007$ & $336.3 \pm 2.8$ & $\bm{6.670 \pm 0.011}$ & $\bm{316.4 \pm 2.2}$ & $-5.9\%$ & $7.398 \pm 0.005$ & $528.9 \pm 4.1$   & $\bm{7.320 \pm 0.009}$ & $\bm{455.6 \pm 6.6}$   & $-13.9\%$\\
        16 & $6.806 \pm 0.013$ & $248.0 \pm 2.7$ & $\bm{6.742 \pm 0.012}$ & $\bm{232.1 \pm 3.7}$ & $-6.4\%$ & $7.341 \pm 0.009$ & $227.7 \pm 3.7$   & $\bm{7.261 \pm 0.011}$ & $\bm{189.5 \pm 1.9}$   & $-16.7\%$\\
        32 & $6.839 \pm 0.017$ & $208.9 \pm 4.4$ & $\bm{6.791 \pm 0.002}$ & $\bm{198.2 \pm 3.9}$ & $-5.1\%$ & $7.327 \pm 0.009$ & $150.0 \pm 1.5$   & $\bm{7.232 \pm 0.009}$ & $\bm{119.3 \pm 0.3}$   & $-20.5\%$\\
        64 & $6.860 \pm 0.005$ & $193.4 \pm 1.7$ & $\bm{6.810 \pm 0.005}$ & $\bm{184.8 \pm 3.2}$ & $-4.5\%$ & $7.331 \pm 0.007$ & $120.9 \pm 1.8$   & $\bm{7.233 \pm 0.011}$ & $\bm{94.6 \pm 0.9}$    & $-21.7\%$\\
       128 & $6.880 \pm 0.010$ & $188.9 \pm 2.6$ & $\bm{6.822 \pm 0.017}$ & $\bm{175.7 \pm 3.7}$ & $-7.0\%$ & $7.356 \pm 0.010$ & $111.7 \pm 1.5$   & $\bm{7.235 \pm 0.007}$ & $\bm{85.1 \pm 1.4}$    & $-23.8\%$\\
       256 & $6.869 \pm 0.010$ & $182.1 \pm 2.7$ & $\bm{6.815 \pm 0.011}$ & $\bm{171.8 \pm 3.2}$ & $-5.7\%$ & $7.371 \pm 0.012$ & $107.7 \pm 1.8$   & $\bm{7.245 \pm 0.013}$ & $\bm{80.2 \pm 0.9}$    & $-25.6\%$\\
       512 & $6.884 \pm 0.013$ & $183.4 \pm 3.4$ & $\bm{6.826 \pm 0.009}$ & $\bm{169.0 \pm 2.4}$ & $-7.9\%$ & $7.378 \pm 0.006$ & $103.9 \pm 1.1$   & $\bm{7.241 \pm 0.011}$ & $\bm{77.4 \pm 1.0}$    & $-25.5\%$\\
      1024 & $6.888 \pm 0.012$ & $181.6 \pm 2.7$ & $\bm{6.818 \pm 0.011}$ & $\bm{167.2 \pm 1.0}$ & $-7.9\%$ & - & - & - & - & -\\
      \midrule
      & \multicolumn{5}{c}{3B uniform diffusion} & \multicolumn{5}{c}{DFM} \\ \midrule \multirow{2}{*}{NFE} & \multicolumn{2}{c}{Standard} & \multicolumn{2}{c}{SHS} & \multirow{2}{*}{$\Delta$} & \multicolumn{2}{c}{Standard} & \multicolumn{2}{c}{SHS} & \multirow{2}{*}{$\Delta$} \\
      \cmidrule(lr){2-3} \cmidrule(lr){4-5} \cmidrule(lr){7-8} \cmidrule(lr){9-10}
      & Entropy & Gen.\ PPL$\downarrow$ & Entropy & Gen.\ PPL$\downarrow$ & & Entropy & Gen.\ PPL$\downarrow$ & Entropy & Gen.\ PPL$\downarrow$ &\\ \midrule
        4 & $6.758$ & $936.3$ & $\bm{6.634}$ & $\bm{869.5}$ & $-7.1\%$ & $6.808 \pm 0.005$ & $263.2 \pm 0.8$ & $\bm{6.808 \pm 0.002}$ & $\bm{261.1 \pm 0.8}$ & $-0.8\%$\\
        8 & $6.660$ & $443.4$ & $\bm{6.536}$ & $\bm{417.2}$ & $-5.9\%$ & $6.252 \pm 0.002$ & $129.7 \pm 0.3$ & $\bm{6.233 \pm 0.006}$ & $\bm{125.7 \pm 0.6}$ & $-3.1\%$\\
       16 & $6.612$ & $221.4$ & $\bm{6.554}$ & $\bm{218.1}$ & $-1.5\%$ & $5.868 \pm 0.004$ & $78.7 \pm 0.2$  & $\bm{5.844 \pm 0.001}$ & $\bm{76.0 \pm 0.1}$  & $-3.4\%$\\
       32 & $6.649$ & $150.2$ & $\bm{6.588}$ & $\bm{142.0}$ & $-5.4\%$ & $5.513 \pm 0.004$ & $51.4 \pm 0.1$  & $\bm{5.459 \pm 0.006}$ & $\bm{48.5 \pm 0.3}$  & $-5.6\%$\\
       64 & $6.654$ & $113.7$ & $\bm{6.619}$ & $\bm{113.0}$ & $-0.6\%$ & $4.895 \pm 0.002$ & $25.4 \pm 0.1$  & $\bm{4.769 \pm 0.043}$ & $\bm{22.4 \pm 0.6}$  & $-12.0\%$\\
      128 & $6.649$ & $104.9$ & $\bm{6.618}$ & $\bm{99.6}$  & $-5.1\%$ & $3.300 \pm 0.004$ & $5.8 \pm 0.0$ & $\bm{3.046 \pm 0.055}$ & $\bm{4.8 \pm 0.1}$ & $-17.3\%$\\
      \bottomrule
\end{tabular}}
\end{table}

\begin{table}
    \centering
    \caption{\textbf{Lexical diversity under NFE budgets} (Distinct-$n$ higher is more diverse, Self-BLEU lower is more diverse).
    Single seed, 1024 samples per cell; metric definitions in Appendix~\ref{appen:details}.} \label{tab:main_diversity}
    \resizebox{\columnwidth}{!}{%
    \begin{tabular}{c cccc cccc|cccc cccc}
      \toprule
      & \multicolumn{8}{c}{\textbf{UDLM}} & \multicolumn{8}{c}{\textbf{GIDD}} \\ \midrule
      \multirow{2}{*}{NFE} & \multicolumn{4}{c}{Standard} & \multicolumn{4}{c}{SHS} & \multicolumn{4}{c}{Standard} & \multicolumn{4}{c}{SHS}\\
      \cmidrule(lr){2-5} \cmidrule(lr){6-9} \cmidrule(lr){10-13} \cmidrule(lr){14-17}
      & D-1 $\uparrow$ & D-2 $\uparrow$ & D-3 $\uparrow$ & S-BLEU $\downarrow$
      & D-1 $\uparrow$ & D-2 $\uparrow$ & D-3 $\uparrow$ & S-BLEU $\downarrow$ & D-1 $\uparrow$ & D-2 $\uparrow$ & D-3 $\uparrow$ & S-BLEU $\downarrow$
      & D-1 $\uparrow$ & D-2 $\uparrow$ & D-3 $\uparrow$ & S-BLEU $\downarrow$ \\
      \midrule
    4 & $0.198$ & $0.758$ & $0.977$ & $0.0000$ & $0.192$ & $0.752$ & $0.975$ & $0.0000$ & $0.251$ & $0.820$ & $0.991$ & $0.0000$ & $0.245$ & $0.814$ & $0.990$ & $0.0000$ \\
         8 & $0.188$ & $0.734$ & $0.964$ & $0.0001$ & $0.181$ & $0.730$ & $0.965$ & $0.0000$ & $0.205$ & $0.745$ & $0.972$ & $0.0000$ & $0.195$ & $0.735$ & $0.971$ & $0.0000$\\
        16 & $0.182$ & $0.719$ & $0.955$ & $0.0001$ & $0.170$ & $0.704$ & $0.951$ & $0.0002$ & $0.180$ & $0.702$ & $0.958$ & $0.0001$ & $0.173$ & $0.692$ & $0.956$ & $0.0001$\\
        32 & $0.177$ & $0.702$ & $0.942$ & $0.0003$ & $0.168$ & $0.695$ & $0.942$ & $0.0002$ & $0.170$ & $0.681$ & $0.950$ & $0.0001$ & $0.159$ & $0.664$ & $0.945$ & $0.0001$\\
        64 & $0.174$ & $0.696$ & $0.937$ & $0.0003$ & $0.170$ & $0.699$ & $0.943$ & $0.0002$ & $0.168$ & $0.675$ & $0.946$ & $0.0001$ & $0.156$ & $0.660$ & $0.943$ & $0.0001$\\
       128 & $0.178$ & $0.700$ & $0.938$ & $0.0003$ & $0.167$ & $0.686$ & $0.935$ & $0.0004$ & $0.168$ & $0.676$ & $0.946$ & $0.0001$ & $0.152$ & $0.654$ & $0.941$ & $0.0001$\\
       256 & $0.176$ & $0.697$ & $0.936$ & $0.0002$ & $0.169$ & $0.691$ & $0.935$ & $0.0003$ & $0.168$ & $0.675$ & $0.945$ & $0.0001$ & $0.156$ & $0.657$ & $0.940$ & $0.0002$\\
       512 & $0.173$ & $0.691$ & $0.932$ & $0.0003$ & $0.166$ & $0.688$ & $0.934$ & $0.0002$ & $0.169$ & $0.679$ & $0.946$ & $0.0001$ & $0.151$ & $0.653$ & $0.939$ & $0.0001$\\
      1024 & $0.175$ & $0.695$ & $0.935$ & $0.0003$ & $0.164$ & $0.682$ & $0.930$ & $0.0003$ & - & - & - & - & - & - & - & -\\ \midrule[1pt]
      & \multicolumn{8}{c}{\textbf{3B uniform diffusion}} & \multicolumn{8}{c}{\textbf{DFM}} \\ \midrule
      \multirow{2}{*}{NFE} & \multicolumn{4}{c}{Standard} & \multicolumn{4}{c}{SHS} & \multicolumn{4}{c}{Standard} & \multicolumn{4}{c}{SHS}\\
      \cmidrule(lr){2-5} \cmidrule(lr){6-9} \cmidrule(lr){10-13} \cmidrule(lr){14-17}
      & D-1 $\uparrow$ & D-2 $\uparrow$ & D-3 $\uparrow$ & S-BLEU $\downarrow$
      & D-1 $\uparrow$ & D-2 $\uparrow$ & D-3 $\uparrow$ & S-BLEU $\downarrow$ & D-1 $\uparrow$ & D-2 $\uparrow$ & D-3 $\uparrow$ & S-BLEU $\downarrow$
      & D-1 $\uparrow$ & D-2 $\uparrow$ & D-3 $\uparrow$ & S-BLEU $\downarrow$ \\
      \midrule
        4 & $0.134$ & $0.676$ & $0.968$ & $0.0000$ & $0.120$ & $0.663$ & $0.968$ & $0.0000$ & $0.055$ & $0.439$ & $0.879$ & $0.0004$ & $0.055$ & $0.436$ & $0.877$ & $0.0005$\\
    8 & $0.102$ & $0.591$ & $0.931$ & $0.0001$ & $0.089$ & $0.579$ & $0.929$ & $0.0001$ & $0.032$ & $0.306$ & $0.781$ & $0.0018$ & $0.031$ & $0.301$ & $0.776$ & $0.0017$\\
    16 & $0.088$ & $0.538$ & $0.896$ & $0.0002$ & $0.079$ & $0.535$ & $0.898$ & $0.0003$ & $0.020$ & $0.220$ & $0.686$ & $0.0047$ & $0.019$ & $0.213$ & $0.676$ & $0.0051$\\
    32 & $0.085$ & $0.522$ & $0.880$ & $0.0003$ & $0.076$ & $0.515$ & $0.881$ & $0.0003$ & $0.013$ & $0.156$ & $0.579$ & $0.0116$ & $0.012$ & $0.147$ & $0.560$ & $0.0138$\\
    64 & $0.083$ & $0.510$ & $0.869$ & $0.0005$ & $0.078$ & $0.510$ & $0.873$ & $0.0004$ & $0.009$ & $0.104$ & $0.432$ & $0.0200$ & $0.008$ & $0.099$ & $0.410$ & $0.0221$\\
    128 & $0.082$ & $0.506$ & $0.863$ & $0.0004$ & $0.076$ & $0.501$ & $0.867$ & $0.0005$ & $0.009$ & $0.075$ & $0.265$ & $0.0236$ & $0.008$ & $0.069$ & $0.237$ & $0.0228$\\
    \bottomrule
    \end{tabular}}
\end{table}

\section{Experiments}
\label{sec:experiments}

\subsection{Text Generation}
\label{sec:exp-text}

We evaluate SHS on four uniform-noise generative language models that span both discrete diffusion (UDLM~\citep{schiff2025simple}, GIDD~\citep{rutte2025generalized}, and a 3B-parameter model~\citep{vonrutte2025scalingbehaviordiscretediffusion}) and
discrete flow matching (DFM)~\citep{gat2024dfm}, with parameter counts ranging from $\sim$110M to
$\sim$3B.
This selection tests whether the variance-reduction benefit
generalizes across model families, training objectives, and scale. In all
cases, we keep the trained model (rate/velocity predictor) \emph{fixed} and
swap only the inference procedure between the standard sampler and SHS.

\paragraph{Protocol.} 
We compare SHS against the standard sampler of Section~\ref{sec:prelim_related} at identical NFE, the number of discretization steps $n$; one step is one forward pass over all positions, so SHS adds no network evaluations.
Generation quality is measured by generative perplexity (Gen.\ PPL), the exponentiated mean token negative log-likelihood that a fixed \texttt{GPT2-large} evaluator assigns to the generated text \citep{schiff2025simple}; lexical diversity by Distinct-$n$ and Self-BLEU.
Each configuration uses $1024$ sentences, averaged over $5$ seeds for UDLM and GIDD and $3$ for DFM; the 3B model uses a single seed.
The maximum NFE is $1024$ for UDLM, $512$ for GIDD, and $128$ for DFM and the 3B model.
For DFM, the model's local rate $\lambda_{ik}$ is converted into a per-step change mass with the Euler form $p_{ik}=\min(h\lambda_{ik},1)$; Appendix~\ref{appen:dfm_discretization} ablates the frozen-rate form $p_{ik}=1-\exp(-h\lambda_{ik})$.
Checkpoints, generation lengths and metric definitions are in Appendix~\ref{appen:details}, and matched samples from both samplers in Appendix~\ref{appen:examples}.

\paragraph{Results.}
SHS lowers Gen.\ PPL at every (model, NFE) budget tested (Table~\ref{tab:main_genppl}), by up to $26\%$ on GIDD and $17\%$ on DFM.
Adding SHS on top of nucleus sampling or temperature lowers Gen.\ PPL further at every setting tested, whereas the adaptive (top-1) scheduling of \citet{vonrutte2025scalingbehaviordiscretediffusion}, which also decides when to edit, is $2$--$5\times$ worse than the standard sampler (Appendix~\ref{appen:baselines}).

The diversity SHS removes is excess that the standard sampler produces over the real distribution.
Table~\ref{tab:main_diversity} shows SHS slightly below the standard sampler on Distinct-$n$; against length-matched real text, the standard sampler is \emph{more} diverse than the corpus it was trained on and SHS moves back toward it, and $88$--$90\%$ of the lost Distinct-1 is token types absent from the corpus.
Fidelity rises at the same time: likelihood under the generating model improves in all seven cells measured, and MAUVE on GIDD rises from $0.381$ to $0.536$ at NFE${=}128$ (Appendix~\ref{appen:diversity}).

\subsection{Jump-Count Variance Diagnostics}
\label{sec:jump_variance}

We directly verify the variance prediction in Proposition~\ref{prop:variance_main}. For each generated sample, position, and NFE budget, we record the realized jump count $J_i$ and the cumulative mass $S_i^{\mathrm{tot}}=\sum_k p_{ik}$. Since the theoretical $1/4$ bound is conditional on fixed $S_i^{\mathrm{tot}}$, we estimate conditional count variance by binning positions according to $S_i^{\mathrm{tot}}$ and averaging the within-bin variance of $J_i$.

\begin{table}[h]
\centering
\small
\caption{
Jump-count variance diagnostics on UDLM. We bin positions by $S_i^{\mathrm{tot}}$ (bin width $0.5$) and report the sample-weighted average of within-bin variance, an estimator of $\mathbb{E}_{S^{\mathrm{tot}}}[\mathrm{Var}(J_i\mid S_i^{\mathrm{tot}})]$.
The standard sampler's empirical variance grows with NFE as the Poisson-binomial prediction $\sum_k p_{ik}(1-p_{ik})$ does; SHS remains below the fixed-mass $1/4$ bound (Proposition~\ref{prop:variance_main}).
}
\vspace{0.5em}
\label{tab:jump_count_variance}
\begin{tabular}{lcccc}
\toprule
NFE
& Standard observed
& \makecell{Standard prediction \\ $\sum_k p_{ik}(1-p_{ik})$}
& SHS observed
& SHS bound \\
\midrule
4   & $0.344$ & $0.405$ & $0.125$ & $0.25$ \\
16  & $1.254$ & $1.642$ & $0.134$ & $0.25$ \\
64  & $2.438$ & $3.096$ & $0.099$ & $0.25$ \\
128 & $3.087$ & $3.845$ & $0.086$ & $0.25$ \\
\bottomrule
\end{tabular}
\end{table}

As shown in Table~\ref{tab:jump_count_variance}, the empirical trend matches the theory: the standard sampler accumulates Poisson-binomial count variance as the grid is refined, whereas SHS sharply concentrates $J_i$ around $\lfloor S_i^{\mathrm{tot}}\rfloor$ or $\lceil S_i^{\mathrm{tot}}\rceil$. This confirms that the proposed intervention removes the intended sampler-induced count variance rather than merely changing aggregate sample statistics.
The concentration is exact: at all $7{,}864{,}320$ positions logged on UDLM and GIDD, $J_i$ is $\lfloor S_i^{\mathrm{tot}}\rfloor$ or $\lceil S_i^{\mathrm{tot}}\rceil$.
The variance SHS removes thus grows with NFE. The Gen.\ PPL gain in Table~\ref{tab:main_genppl} need not track this growth, because a uniform-noise model can revise a token written incorrectly at an earlier step and the share of a count deviation that this self-correction absorbs varies by model: the gain grows with NFE on GIDD and DFM and stays within $4.2$--$7.9\%$ on UDLM and $0.6$--$7.1\%$ on the 3B model.
Drawing $\theta_i$ independently per position is what matters: sharing one $\theta$ across positions at the same edit budget raises UDLM Gen.\ PPL at NFE${=}128$ from $174.2$ to $213.5$, while redrawing it at each unit interval leaves it unchanged ($173.7$; Appendix~\ref{appen:closedloop}).

\subsection{Under- and Over-Editing as Count-Deviation Failures}
\label{sec:count_deviation_quality}

We next test whether deviations from the model-implied cumulative mass are associated with local quality degradation. For each token position, define the count deviation
\[
\Delta J_i := J_i - \mathrm{round}(S_i^{\mathrm{tot}}),
\]
where $\mathrm{round}(S_i^{\mathrm{tot}})$ is the nearest integer to the cumulative mass $S_i^{\mathrm{tot}}=\sum_k p_{ik}$.
We then bin tokens by $\Delta J_i$ and measure token-level evaluator NLL.

\begin{figure}[H]
\centering
\includegraphics[width=\linewidth]{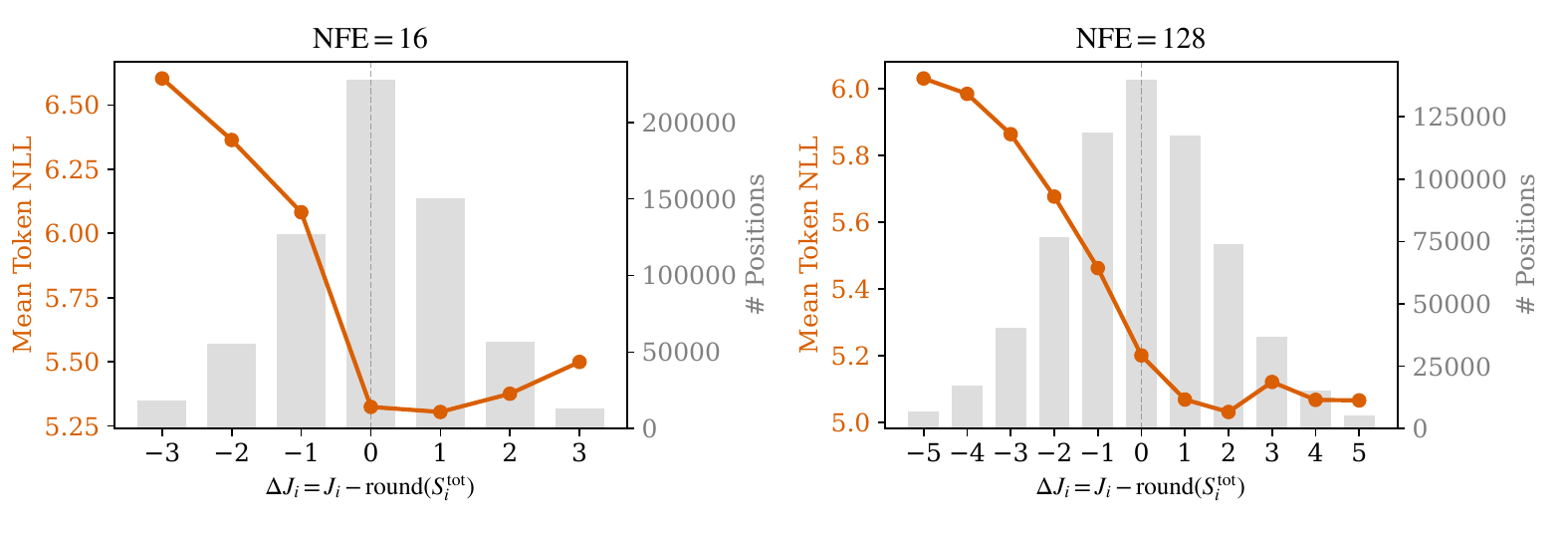}
\caption{
Token-wise NLL versus count deviation $\Delta J_i = J_i - \mathrm{round}(S_i^{\mathrm{tot}})$ on UDLM (Standard sampler) at NFE${=}16$ (left) and NFE${=}128$ (right). Orange line: mean token NLL per deviation bin (left axis). Gray bars: position count per bin (right axis).
}
\label{fig:count_deviation_quality}
\end{figure}

Figure~\ref{fig:count_deviation_quality} shows that token NLL is lowest near $\Delta J_i=0$ and rises steeply as the realized count falls below the model-implied cumulative mass.
The degradation is far more pronounced on the under-edit side ($\Delta J_i<0$) than on the over-edit side ($\Delta J_i>0$) at matched $|\Delta J_i|$, and persists within every $S_i^{\mathrm{tot}}$ stratum at both NFE budgets (Appendix~\ref{appen:stratified}).
SHS confines $\Delta J_i$ to $\{-1,0,1\}$ and thus cuts off both tails.

\section{Related Work}

\paragraph{Sampling rules.}
A growing line of work improves discrete diffusion/flow samplers by modifying the inference procedure itself. The most lightweight route is destination-distribution shaping via temperature, top-$k$, or nucleus sampling, which reshapes the destination kernel $q_{i,k}$ at each step \citep{fan2018hierarchical,holtzman2020curious}. JYS optimizes the timestep allocation under a fixed sampling budget \citep{park2025jys}. SCUD conditions discrete diffusion models on the jump schedule, and Neural CTMC explicitly separates the reverse process into an exit-rate component for when to jump and a jump distribution for where to jump \citep{amin2025scud,li2026neuralctmc}. Predictor-corrector methods, informed correctors, and $\Psi$-samplers add or modify correction dynamics to improve sampling \citep{zhao2025informed,deschenaux2026psisamplers}. Planner-based samplers, including DDPD and P2, modify which positions are updated \citep{liu2025ddpd,peng2025p2}; related first-hitting and token-ordering analyses exploit the order/time structure of masked diffusion \citep{zheng2025fhs,kim2025tokenordering}. Few-step and scaling-oriented methods such as FS-DFM and Duo change the training problem, introduce step-budget conditioning, or distill a new sampler \citep{monsefi2025fs,sahoo2025duo}. Deterministic Discrete Denoising derandomizes the reverse Markov chain by replacing the stochastic denoising transition itself \citep{suzuki2025ddd}. In contrast, SHS preserves the learned change masses and destination kernels and modifies only the temporal coupling of the binary stay-vs.-replace events through a single per-position phase, with $O(N)$ overhead and no extra forward passes, hyperparameters, or retraining.

\paragraph{Variance reduction in event simulation.}
Classical exact simulation of non-homogeneous Poisson processes uses thinning \citep{lewis1979simulation,ogata1981lewis}, and CTMC inference often uses uniformization to convert continuous-time jump processes into conditionally discrete transition problems \citep{rao2013mjp,vandijk2018uniformization}. On the variance-reduction side, systematic and stratified resampling reduce the variance of Monte Carlo resampling relative to independent multinomial sampling in sequential Monte Carlo \citep{douc2005comparison}, and randomized rounding converts a real-valued mass into an integer count while preserving its expectation \citep{raghavan1987randomized,owen2013monte}.

\section{Conclusion}
\label{conclusion}
We diagnosed a fundamental instability in step-based simulation of CTMC/DTMC discrete diffusion and flow models with a (stay vs.\ replace) decomposition: the Poisson-binomial jump-count variance grows with NFE toward the cumulative mass $S$, inducing under-edit and over-edit failure modes. To address this, we proposed \emph{Systematic Hazard Sampling} (SHS), a drop-in rule that bounds this variance below $1/4$ uniformly in NFE through a single random phase per position, while preserving the per-step event marginals and destination kernel of the standard sampler, without retraining, additional hyperparameters, or extra computation.
We confirm this empirically on diverse uniform-noise backbones spanning $\sim$110M to $\sim$3B parameters and covering both diffusion- and flow-based families (UDLM, GIDD, DFM, and a 3B uniform diffusion model): SHS consistently improves generation quality, as reflected in Gen.\ PPL, and composes with inference-time methods that act on the destination distribution.
We hope that SHS contributes to building more robust discrete diffusion and flow models, and that it encourages further research on these models as a promising direction for non-autoregressive generation.

\paragraph{Limitations.}
SHS is an inference-time variance reduction integrator rather than an exact simulator of the underlying state-dependent CTMC/DTMC process: the per-step change masses $p_{ik}(x)$ are state-dependent, so the realized mass $S_i^{\mathrm{tot}}$ is itself random and its distribution may differ between the two samplers, and Propositions~\ref{prop:unbiased_main}--\ref{prop:variance_main} are accordingly fixed-mass and fixed-event statements.
SHS also preserves the expected jump count and reduces only its variance; since the quality cost of count deviation lies on the under-edit side (Section~\ref{sec:count_deviation_quality}), biasing the jump count itself is a separate design choice that we do not explore.
Our experiments cover uniform-noise discrete diffusion language models, where multi-edit dynamics make jump-count variance most consequential, with a single seed for the 3B model; larger code benchmarks remain future work.

\clearpage

\begin{ack}
Seunghwan Jang would like to thank the KAIST ACSS Lab for its support of this work.
We thank Albert No for valuable suggestions and discussions throughout this project, and Yair Schiff for helpful discussions and for sharing useful resources.
This work received no specific funding, and the authors declare no competing interests.
\end{ack}

\bibliography{ref}
\bibliographystyle{plainnat}

\newpage
\appendix

\section{Visual Diagnostics on Quantized MNIST}\label{appen:MNIST}

\subsection{Setup}
\label{sec:exp-setup}
Our goal is to isolate the effect of event-time sampling on a continuous-time discrete generative model.
We fix the pretrained predictor and vary only the sampling procedure: (i) the standard sampler, and (ii) our Systematic Hazard Sampling (SHS).

\vspace{-10px}

\paragraph{Few-step regime.}
To stress-test trajectory instability, we vary the number of discretization steps
$\mathrm{NFE} \in \{256, 64, 16, 8\}$ while keeping the same model and schedule.
Figure~\ref{fig:exp-jump-locations} reports hazard-space jump locations at $\mathrm{NFE}=256$; Figure~\ref{fig:exp-mnist-recon} shows generated samples across all four budgets.

\begin{figure}[H]
\vspace{-10px}
  \centering
  \begin{subfigure}{\linewidth}
    \centering
    \includegraphics[width=\linewidth]{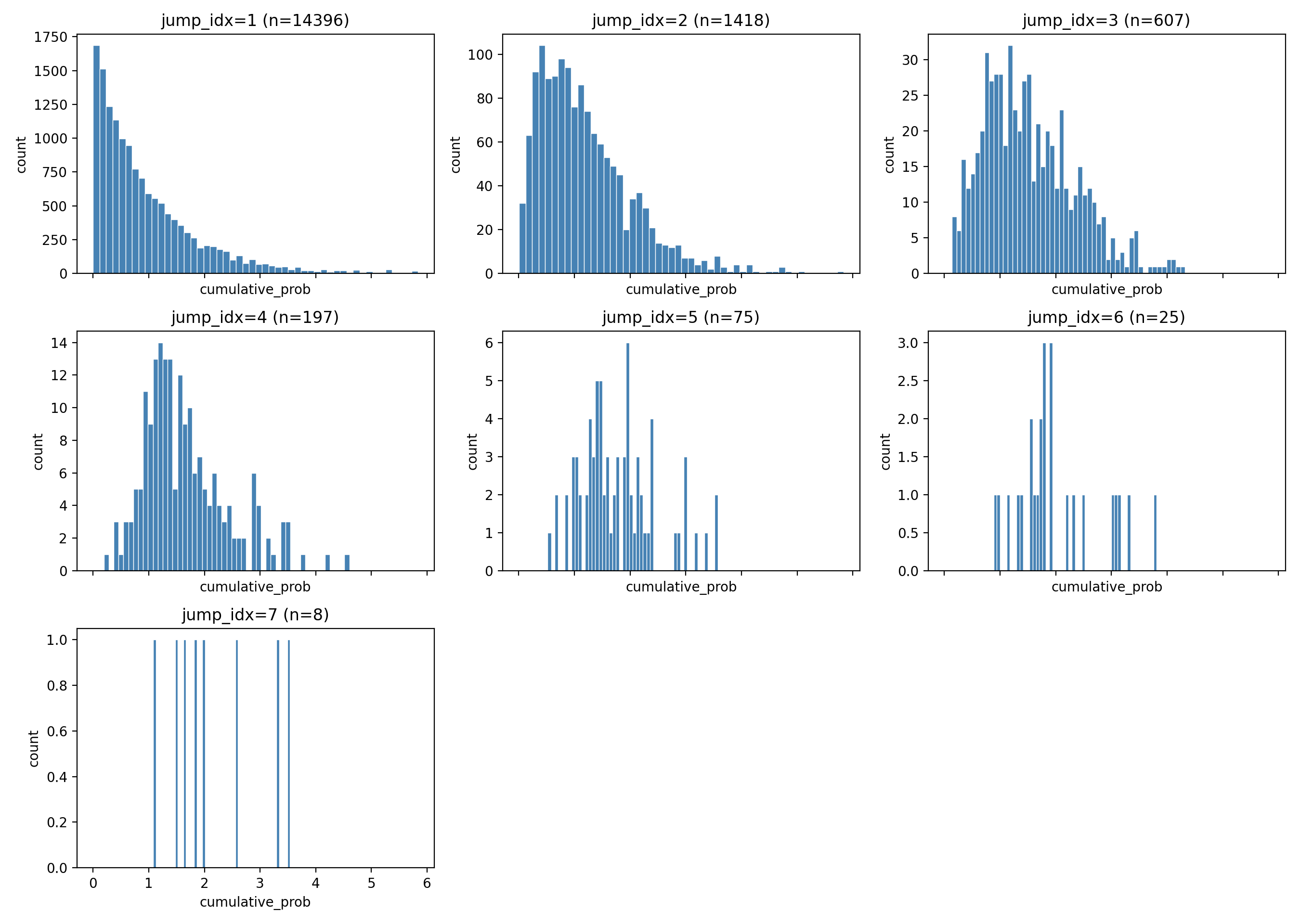}
    \caption{Standard sampling}
    \label{fig:hist-standard}
  \end{subfigure}
  \begin{subfigure}{\linewidth}
    \centering
    \includegraphics[width=\linewidth]{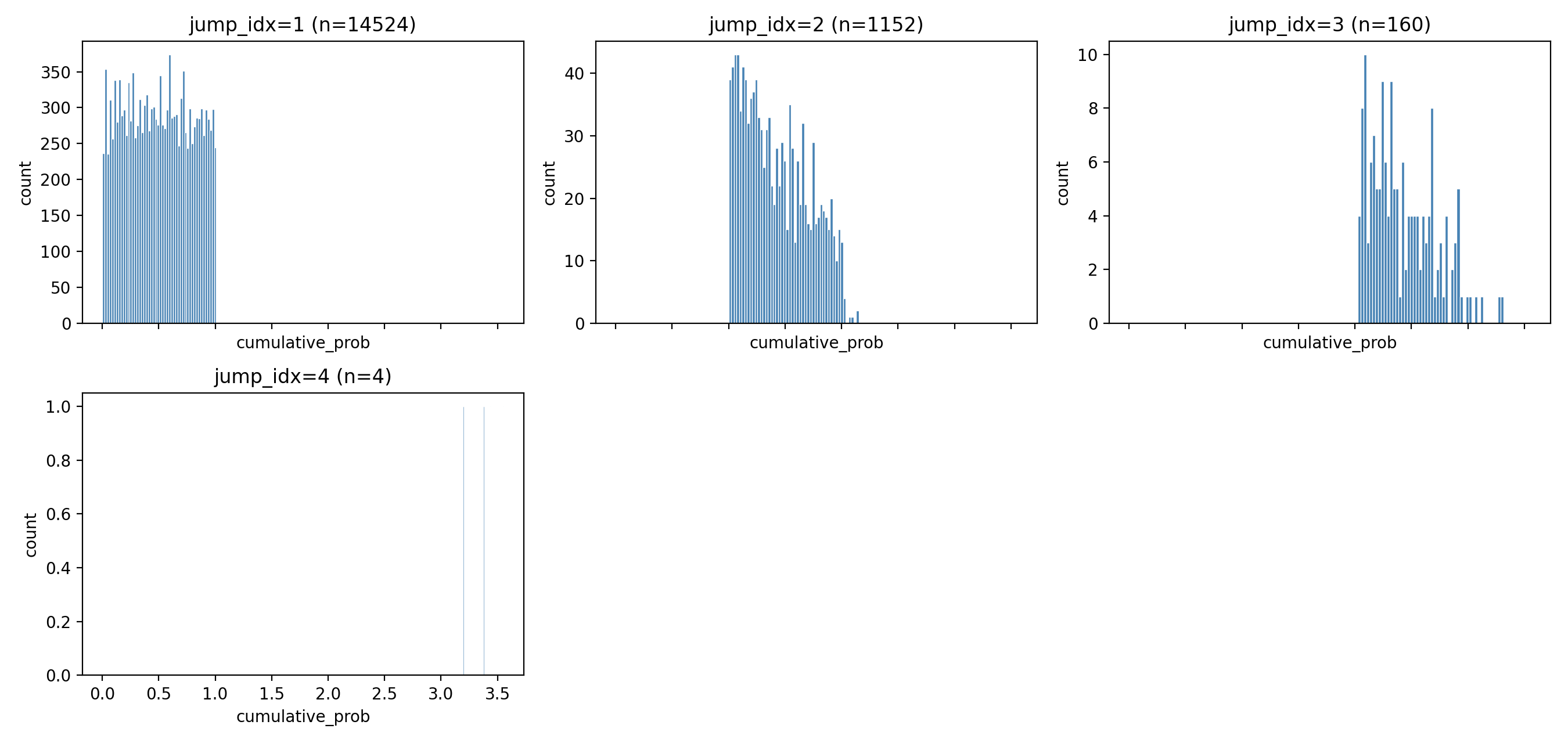}
    \caption{SHS}
    \label{fig:hist-shs}
  \end{subfigure}
  \caption{\textbf{Hazard-space jump locations} ($\mathrm{NFE}=256$). Histograms of $S_k$, the cumulative hazard at the $k$-th jump.
  Standard sampling exhibits Erlang/Gamma-shaped variability, while SHS produces systematic, bounded-support locations,
  demonstrating reduced trajectory randomness.}
  \label{fig:exp-jump-locations}
  \vspace{-15px}
\end{figure}

\subsection{Jump Timing Diagnostics in Hazard Space}
\label{sec:exp-jump-loc}
We first validate whether SHS produces the intended \emph{regularization of event times} without altering the
per-jump categorical choice.

For a single sampled trajectory, let $\lambda(t)$ denote the (token-wise) jump rate and define the cumulative hazard
\begin{equation}
S(t) = \int_0^t \lambda(s)\,ds.
\end{equation}
We omit the position index $i$ used in the main text since this section examines a single trajectory.
We record the hazard-space jump locations $\{S_k\}_{k\ge 1}$, where $S_k$ is the cumulative hazard value at the
$k$-th jump.

\paragraph{Standard baseline: Erlang/Gamma-shaped locations.}
Under the standard non-homogeneous Poisson process (NHPP) view, inter-event increments in hazard space satisfy $\Delta S_k \sim \mathrm{Exp}(1)$ i.i.d.,
so the $k$-th jump location becomes
\begin{equation}
S_k = \sum_{j=1}^k \Delta S_j \sim \mathrm{Gamma}(k,1),
\end{equation}
which yields an Erlang-shaped distribution as $k$ increases. Empirically, we observe broad, heavy-tailed variability in early jump locations under standard sampling, consistent with the Poisson-binomial fluctuations analyzed in the main text (Section~\ref{prelim}) under fine grids approaching the Poisson limit.

\paragraph{SHS: systematic (bounded-support) jump locations.}
By construction (Algorithm~\ref{alg:shs}), SHS triggers events at unit-spaced thresholds in cumulative hazard space, constraining the $k$-th event to a narrow hazard interval.
Consequently, the empirical histograms of $S_k$ concentrate with bounded support and avoid the exponential long-tail behavior.
This confirms that SHS makes the \emph{timing} of jumps substantially more regular while leaving the per-jump destination
sampling unchanged.

\subsection{Quantized MNIST Generation: Eliminating Under-Edit}
\label{sec:exp-mnist-recon}
We next visualize the practical effect of trajectory stabilization on a simple discrete generation task.
Starting from a noise initialization, we run the sampler for $\mathrm{NFE} \in \{256,64,16,8\}$ discretization steps and compare the
resulting samples.

\paragraph{Observation.}
Standard sampling often exhibits an \emph{under-edit} failure mode in the few-step regime:
some positions experience too few jumps, leaving noticeable $x_0$-derived noise even late in sampling.
In contrast, SHS produces visibly cleaner samples at early stages and substantially reduces residual noise in low-NFE
runs. Qualitatively, SHS maintains legible digit structure even at $\mathrm{NFE}=16$ and $\mathrm{NFE}=8$, where Standard outputs remain dominated by
unresolved noise.

\vspace{-0.4em}
\begin{figure}[H]
  \centering
  
  \begin{minipage}{\linewidth}\centering
    \begin{minipage}{0.24\linewidth}\centering {\scriptsize $\mathrm{NFE}=256$}\end{minipage}\hfill
    \begin{minipage}{0.24\linewidth}\centering {\scriptsize $\mathrm{NFE}=64$}\end{minipage}\hfill
    \begin{minipage}{0.24\linewidth}\centering {\scriptsize $\mathrm{NFE}=16$}\end{minipage}\hfill
    \begin{minipage}{0.24\linewidth}\centering {\scriptsize $\mathrm{NFE}=8$}\end{minipage}
  \end{minipage}
  
  \vspace{0.2em}
  
  \begin{subfigure}{\linewidth}\centering
    \begin{minipage}{0.24\linewidth}\centering\halfimg{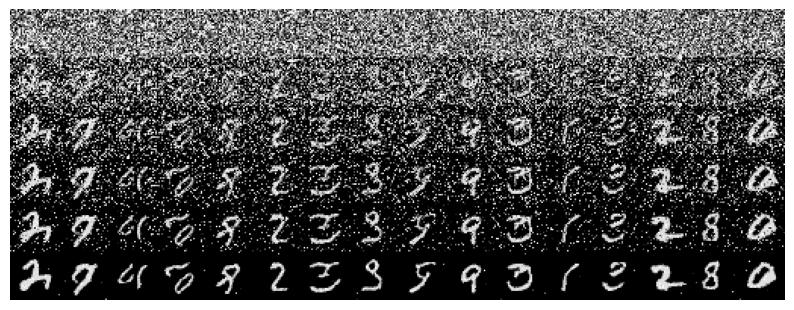}\end{minipage}\hfill
    \begin{minipage}{0.24\linewidth}\centering\halfimg{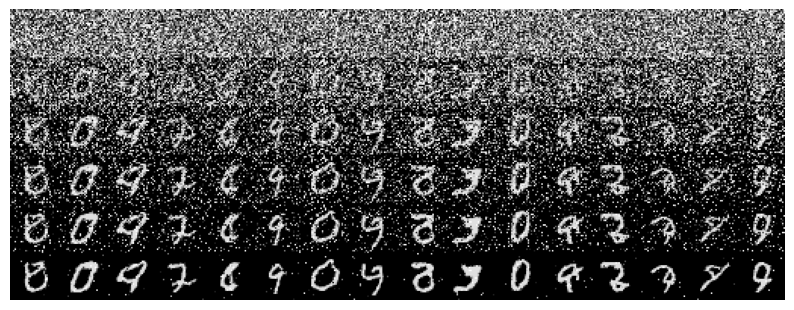}\end{minipage}\hfill
    \begin{minipage}{0.24\linewidth}\centering\halfimg{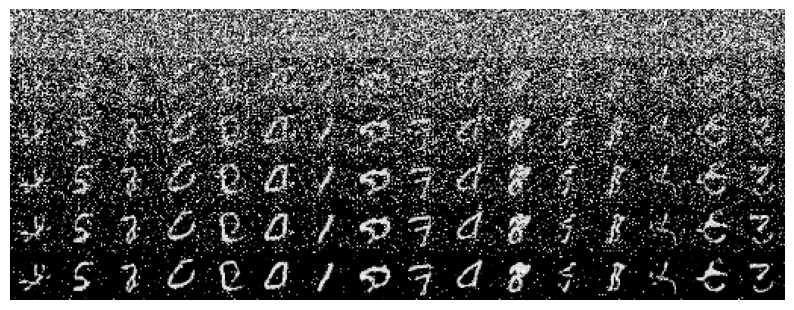}\end{minipage}\hfill
    \begin{minipage}{0.24\linewidth}\centering\halfimg{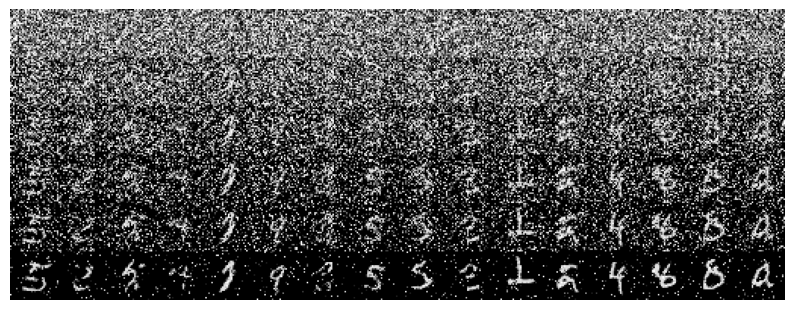}\end{minipage}
    \vspace{-0.5em}
    \caption{Standard sampling}
    \label{fig:recon-standard}
  \end{subfigure}
  
  \vspace{0em}
  
  \begin{subfigure}{\linewidth}\centering
    \begin{minipage}{0.24\linewidth}\centering\halfimg{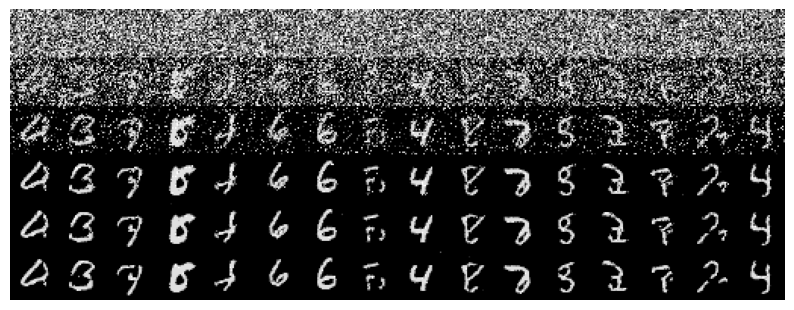}\end{minipage}\hfill
    \begin{minipage}{0.24\linewidth}\centering\halfimg{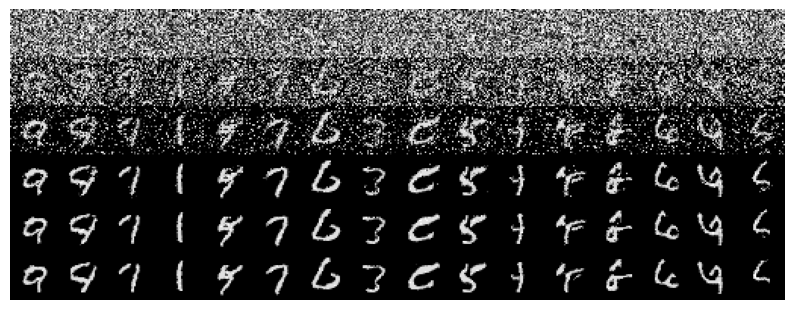}\end{minipage}\hfill
    \begin{minipage}{0.24\linewidth}\centering\halfimg{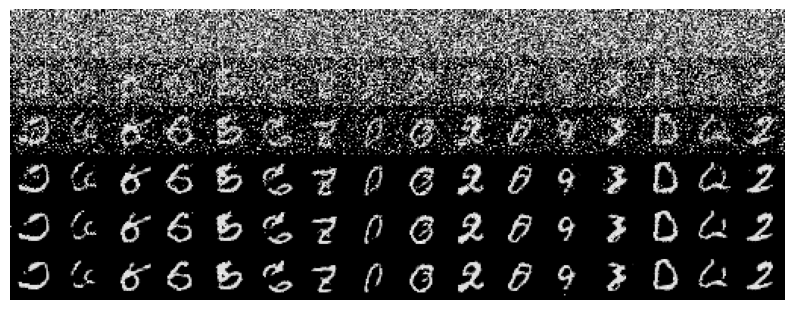}\end{minipage}\hfill
    \begin{minipage}{0.24\linewidth}\centering\halfimg{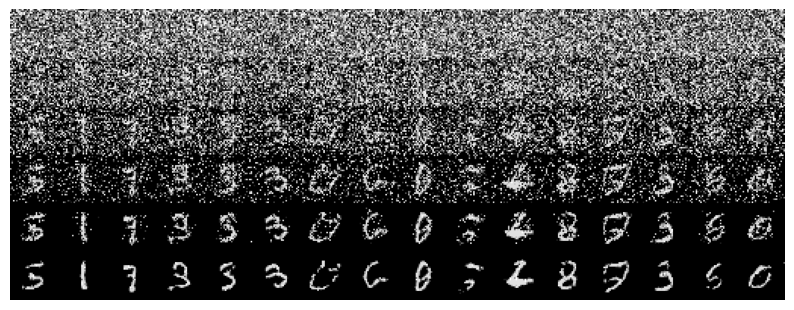}\end{minipage}
    \vspace{-0.5em}
    \caption{SHS}
    \label{fig:recon-shs}
  \end{subfigure}

  \vspace{-0.6em}
  \caption{\textbf{Quantized MNIST samples under varying step budgets.}
  Standard sampling (top) vs.\ SHS (bottom) at $\mathrm{NFE} \in \{256, 64, 16, 8\}$ discretization steps.
  SHS reduces the under-edit regime and yields cleaner early-stage samples.}
  \label{fig:exp-mnist-recon}
\end{figure}

\newpage

\section{Proofs and Details for Section~\ref{subsec:theory}}
\subsection{Proof of Proposition~\ref{prop:unbiased_main} (Jump-Count Unbiasedness)}
\label{appen:unbiased_proof}
Fix $S\ge 0$ and draw $\theta\sim \mathrm{Uniform}(0,1)$.
Let $I=\lfloor S\rfloor$ and $f = S - I \in [0,1)$.
Define
\begin{equation}
J \;=\; I + \mathbf{1}[\theta<f].    
\end{equation}
Then
\begin{equation}
\mathbb{E}[J] \;=\; I + \mathbb{P}(\theta<f) \;=\; I + f \;=\; S,
\end{equation}
since $\theta$ is uniform on $[0,1]$.
This proves Proposition~\ref{prop:unbiased_main}.

We emphasize again that Proposition~\ref{prop:unbiased_main} is a statement about the integer jump count \emph{conditional on a fixed cumulative mass $S$}.
It does not imply that SHS is distributionally unbiased with respect to the original state-dependent CTMC/DTMC sampler.

\subsection{Proof of Proposition~\ref{prop:variance_main} (Minimal Variance)}
\label{appen:variance_proof}
Using the notation from Appendix~\ref{appen:unbiased_proof}, we have $J=I+B$ where $B=\mathbf{1}[\theta<f]\sim\mathrm{Bernoulli}(f)$.
Since $I$ is deterministic,
\begin{equation}
\mathrm{Var}(J)=\mathrm{Var}(I+B)=\mathrm{Var}(B)=f(1-f)\le \frac{1}{4},
\end{equation}
where the last inequality follows because the quadratic $f(1-f)$ attains its maximum at $f=1/2$.

We now show that this is also the minimum variance among all $\mathbb{Z}_{\ge 0}$-valued unbiased estimators of $S$. 
Let $\widetilde{J}$ be any $\mathbb{Z}_{\geq 0}$-valued random variable with $\mathbb{E}[\widetilde{J}]=S$.
Write $U:=\widetilde{J}-I\in\mathbb{Z}$, so that $\mathbb{E}[U]=f$ and $\mathrm{Var}(\widetilde{J})=\mathbb{E}[(\widetilde{J}-S)^2]=\mathbb{E}[(U-f)^2]$.
For every integer $U$ the identity $(U-f)^2=U(U-1)+(1-2f)U+f^2$ holds, and taking expectations gives
\[
\mathrm{Var}(\widetilde{J})=\mathbb{E}[U(U-1)]+(1-2f)f+f^2=\mathbb{E}[U(U-1)]+f(1-f).
\]
Since $U(U-1)\ge0$ for every integer $U$, with equality exactly when $U\in\{0,1\}$, we obtain $\mathrm{Var}(\widetilde{J})\ge f(1-f)$, with equality if and only if $\widetilde{J}\in\{I,I+1\}$ almost surely.
In that case the unbiasedness constraint $\mathbb{E}[\widetilde{J}]=I+f$ forces
\[
	\mathbb{P}(\widetilde{J}=I+1)=f,\qquad \mathbb{P}(\widetilde{J}=I)=1-f,
\]
which is the law of $J$.
Thus the minimum possible variance is $f(1-f)\le1/4$ for \emph{any} such unbiased estimator, and SHS attains this minimum.
This proves Proposition~\ref{prop:variance_main}.

\section{Proofs and Details for Section~\ref{sec:multimodality}}
\label{appen:multimodality}

\subsection{Proof of Proposition~\ref{prop:event_marginal}}
Let $S_{i,k}{\,:=\,}\sum_{j<k}p_{ij}$ and define the residual threshold
$R_{ik}{\,:=\,}(\theta_i-S_{i,k})\bmod 1.
$
Since $\theta_i\sim\mathrm{Uniform}(0,1)$ and the modular shift preserves uniformity, $R_{ik}$ is also uniform on $[0,1)$. SHS triggers an edit at step $k$ exactly when $S_{i,k}+p_{ik}$ crosses a unit-spaced threshold $\theta_i+m$, equivalently when $R_{ik}<p_{ik}$. Hence
\[
\mathbb{P}_{\theta_i}\!\left(B^{\mathrm{SHS}}_{ik}=1\,\middle|\,\{p_{ij}\}\right)
=\mathbb{P}(R_{ik}<p_{ik})=p_{ik},
\]
identical to the standard Bernoulli$(p_{ik})$ marginal, proving~\eqref{eq:event_marginal_preservation}.

Since $J_i^s := \sum_{k=0}^{n-1} B_{ik}^s$ is the total number of jump events at position $i$ under sampler $s\in\{\mathrm{SHS},\mathrm{std}\}$, the expectation $\mathbb{E}[J_i^{\mathrm{SHS}}\mid\{p_{ij}\}]=\sum_k p_{ik}$ follows by linearity of expectation, proving~\eqref{eq:terminal_count_expectation_preservation}.

\subsection{Proof of Proposition~\ref{prop:destination_kernel_identity}}

The event scheduler determines only $B_{ik}\in\{0,1\}$. Once $B_{ik}=1$, both
the standard sampler and SHS draw the replacement token from
$q_{i,k}(\cdot\mid X_{t_k})$. This gives
\[
\mathbb{P}^s\!\left(
X^i_{t_{k+1}}=v
\,\middle|\,
B_{ik}=1,\; X_{t_k}=x
\right)
=
q_{i,k}(v\mid x)
\]
for both samplers. 
Note that given an externally specified event pattern $b$, both samplers start from the same initial law $p_0$, keep positions unchanged when $b_{ik}=0$, and draw from the same destination kernel $q_{i,k}(\cdot\mid X_{t_k})$ when $b_{ik}=1$. 
Thus, recursively applying this same update rule over all steps and
positions yields the same fixed-event trajectory kernel $\mathcal K_b$.

\subsection{Count decomposition and Lipschitz propagation}
\label{appen:lipschitz}

A second application of the law of total variance to the count vector
$C(B)=(J_1,\ldots,J_N)$ refines the event-pattern variability term in
\eqref{eq:totvar_decomp}:
\begin{equation}
\mathrm{Var}_{B\sim\pi_s}[\mu_F(B)]
=
\underbrace{
\mathrm{Var}_C\!\left[
\mathbb{E}_{B\sim\pi_s}[\mu_F(B)\mid C]
\right]
}_{\text{count-driven}}
+
\underbrace{
\mathbb{E}_C\!\left[
\mathrm{Var}_{B\sim\pi_s}(\mu_F(B)\mid C)
\right]
}_{\text{within-count timing/order}}.
\label{eq:count_timing_decomp}
\end{equation}
Proposition~\ref{prop:variance_main} controls the first term through the distribution of the count vector. 
For sampler $s$, let
\[
\bar\mu_{F,s}(c):=\mathbb{E}_{B\sim\pi_s}[\mu_F(B)\mid C=c].
\]
If $\bar\mu_{F,s}$ admits an $L_C$-Lipschitz extension from integer count vectors $c$ to $\mathbb{R}_{+}^N$,
then for fixed cumulative mass vector
$S=(S_1^{\mathrm{tot}},\ldots,S_N^{\mathrm{tot}})$,
\begin{equation}
\mathrm{Var}_C[\bar\mu_{F,s}(C)] \le \mathbb{E} \left[ \bigl(\bar\mu_{F,s}(C)-\bar\mu_{F,s}(S)\bigr)^2\right] \leq L_C^2\,\mathbb{E}\!\left[\|C-S\|_2^2\right].
\label{eq:lipschitz_count_bound}
\end{equation}
When the scheduler is count-unbiased for the fixed mass vector $S$, i.e., $\mathbb{E}[C]=S$, we have $\mathbb{E}[J_i]=S_i^{\mathrm{tot}}$ for each coordinate $i$.
Thus, the right side of~\eqref{eq:lipschitz_count_bound} is equal to
\[
    L_C^2\,\mathbb{E}\!\left[\|C-S\|_2^2\right] = L_C^2\sum_{i=1}^{N} \mathrm{Var}(J_i\mid S_i^{\mathrm{tot}}).
\]
For SHS, Proposition~\ref{prop:variance_main} gives
$\mathrm{Var}(J_i\mid S_i^{\mathrm{tot}})\le1/4$, so the count-driven component
is bounded by $L_C^2N/4$.

\subsection{Scope under state-dependent rates}

In the full sampler, the change masses $p_{ik}=p_{ik}(X_{t_k})$ are
state-dependent. Therefore the observational conditional law
$\mathbb{P}^s(X_{0:n}\mid B=b)$ need not equal the fixed-event kernel
$\mathcal K_b$. 
This is because, in the observational sampler, the event
$B=b$ also encodes information about the visited states through the state-dependent
masses $p_{ik}(X_{t_k})$, whereas $\mathcal{K}_b$ treats $b$ as fixed and given.
The exact structural statements used in
Section~\ref{sec:multimodality} are fixed-path event-marginal preservation,
destination-kernel identity, and fixed-event trajectory-kernel identity. The
empirical separation between jump-count variance reduction and corpus-level diversity is evaluated directly in Section~\ref{sec:exp-text} and Appendix~\ref{appen:diversity}.

\section{Additional Analysis of Systematic Hazard Sampling}
\label{app:theory}

This appendix collects additional properties of SHS
and clarifies its scope for CTMC/DTMC-based samplers.

\subsection{SHS as a systematic coupling of Bernoulli trials}
\label{sec:systematic_coupling}

Consider a fixed sequence of per-step change masses $\{p_k\}_{k=0}^{n-1}\subset[0,1]$ and define the cumulative sums
$S_k=\sum_{j=0}^{k-1}p_j$ (with $S_0=0$) and $S_n=\sum_{j=0}^{n-1}p_j$.
The standard step-based sampler draws independent indicators $B_k^{\mathrm{iid}}\sim \mathrm{Bernoulli}(p_k)$ and the total number of changes is
$J=\sum_k B_k^{\mathrm{iid}}$ (Poisson-binomial).

SHS instead draws a \emph{single} phase $\theta\sim\mathrm{Uniform}(0,1)$.
Write $S_k = I_k + f_k$ with $I_k := \lfloor S_k \rfloor$ and $f_k := S_k - I_k \in [0,1)$,
and note that $S_{k+1}=S_k+p_k$ (so $S_{k+1}=I_{k+1}+f_{k+1}$ with $I_{k+1}:=\lfloor S_{k+1}\rfloor$ and $f_{k+1}:=S_{k+1}-I_{k+1}$).
We define a coupled indicator sequence by
\begin{equation}
B_k^{\mathrm{shs}}
\;=\;
\mathbf{1}\!\left[\;
I_k + \mathbf{1}[\theta<f_k]
\;<\;
I_{k+1} + \mathbf{1}[\theta<f_{k+1}]
\;\right],
\label{eq:shs_indicator}
\end{equation}
which is equivalent to triggering a change whenever the cumulative sum crosses the next unit boundary $\theta+m$.
Summing \eqref{eq:shs_indicator} over $k$ yields the randomized-rounding form
$J^{\mathrm{shs}}=\lfloor S_n\rfloor+\mathbf{1}[\theta<f_n]$ (cf.~\eqref{eq:shs_rounding}).
We emphasize that this algebraic form serves to make the marginal-preservation and randomized-rounding properties more explicit.

\subsection{Event-time systematic placement in cumulative hazard space}
\label{sec:hazard_strat}

For a CTMC with an \emph{exogenous} (state-independent) rate $\lambda(t)\ge 0$, define the continuous cumulative hazard
$S(t)=\int_0^t \lambda(\tau)\,d\tau$.
SHS places event boundaries at $\theta,\,\theta+1,\,\theta+2,\dots$ in hazard space, so the $m$-th event time is
\begin{equation}
T_m \;=\; \inf\{t\in[0,1]: S(t)\ge \theta+(m-1)\}.
\end{equation}
Equivalently, the hazard value at event $m$ is exactly $\theta+(m-1)$.
Thus, SHS guarantees \emph{exactly one} event in each unit-length hazard interval, eliminating the exponential-tail inter-arrival variability of Poisson processes
in hazard space (where $\Delta S\sim \mathrm{Exp}(1)$).

On a discrete grid (DTMC kernels or CTMC $\tau$-leaping), $S_k$ increases in steps.
If each increment satisfies $p_k\le p_{\max}\le 1$, then the boundary-crossing event triggered at step $k$ overshoots its hazard boundary by at most $p_{\max}$.
Therefore, SHS still yields tightly controlled event locations in hazard/jump-mass space even under discretization.
\paragraph{Timing of individual edits.}
SHS moves edits slightly earlier in hazard space and makes their timing less variable.
In cumulative-hazard space the $m$-th edit occurs at $\theta+(m-1)$, with mean $m-\tfrac12$ and variance $\tfrac1{12}$, against $\mathrm{Gamma}(m,1)$ with mean $m$ and variance $m$ under the standard sampler.
Conditionally on no edit having been triggered yet, the probability of an edit at the next step rises (with $p_1=0.8$ and $p_2=0.15$, it is $0.15/0.2=0.75$ after a step without an edit), but this state is reached with probability $0.2$, so the per-step marginal remains $0.2\times0.75=p_2$ (Proposition~\ref{prop:event_marginal}); this coupling across steps is what removes the count variance.
On UDLM, the mean jump count per position under SHS exceeds the standard sampler's by $2.1$/$2.7$/$3.1\%$ at NFE $16$/$64$/$128$, yet the final-step edit rate falls to $0.43$/$0.15$/$0.11\times$: SHS edits more often than the standard sampler up to $k/n\approx0.75$--$0.81$ and less often afterwards, independently of $n$, so edits are redistributed toward the front of the trajectory.

\subsection{Scope: inference-time variance reduction vs.\ exact simulation}
\label{sec:scope}

The results in Propositions~\ref{prop:unbiased_main} and~\ref{prop:variance_main} are unconditional statements about randomized rounding for a \emph{fixed} mass $S$ (or a fixed sequence $\{p_k\}$).
In a full generative model, the per-step masses $p_{ik}(x)$ (or rates $\lambda_i(t,x)$) are typically \emph{state-dependent}.
Because SHS couples change decisions over time through the phase $\theta_i$, it introduces additional memory beyond the visible state $x$.
Accordingly, SHS should be viewed as an \emph{inference-time variance reduction integrator} for CTMC/DTMC samplers rather than as an exact simulator of the underlying Markov process.

Despite this, SHS leaves the per-change destination kernels $q_{i,k}$ unchanged and dramatically reduces sampler-induced variability in the number and timing of edits.
This reduction is especially valuable under uniform-noise initialization, where meaningful generation requires multiple self-correction edits per position and where Poisson-binomial fluctuations can dominate the observed instability.

\subsection{Absorbing-state (mask-start) implies single-jump trajectories}
\label{app:maskstart_single_edit}
Here, we formalize why the jump-count variance discussed in Section~\ref{sec:prelim_related} is a
\emph{multi-edit} phenomenon and is absent under the standard absorbing-state (mask-start) setting.

Let $\mathcal{V}$ be the vocabulary and let $m=\texttt{[MASK]}$ be a dedicated sentinel token.
Define the extended alphabet $\bar{\mathcal{V}} := \mathcal{V} \cup \{m\}$, and consider step-based sampling on a grid
$\{t_k\}_{k=0}^{n}$.
A state is $x \in \bar{\mathcal{V}}^N$.
At each step $k$ and position $i$, assume the sampler admits a (stay vs.\ replace) decomposition
as in Section~\ref{prelim}, with a change mass $p_{ik}(x)\in[0,1]$ and a destination distribution
$q_{i,k}(\cdot\mid x)\in\Delta(\mathcal{V})$.

\begin{assumption}[Unmask-only (Absorbing) Dynamics]\label{assum:unmask_only}
    The absorbing-state (mask-start) sampler abides by the following \emph{unmask-only} dynamics: once a position leaves $m$, it never changes again:
    \begin{equation}
    p_{ik}(x)=0 \quad \text{whenever } x_i\in \mathcal{V}, \text{ for any }x\in\bar{\mathcal{V}}^N.
    \label{eq:maskstart_absorb_condition}
    \end{equation}
\end{assumption}
%

\begin{proposition}[Single-edit property under absorbing-state dynamics]
\label{prop:maskstart_single_edit}
Under Assumption~\ref{assum:unmask_only}, each position can undergo at most one edit:
for every $i$, $J_i\in\{0,1\}$ almost surely.
\end{proposition}

\begin{proof}
If $J_i=0$, the statement is obvious.
Otherwise, define the (non-trivial) edit indicator and total jump count
\begin{equation}
B_{ik} := 1[x'_{i}\neq x_i], \qquad
J_i := \sum_{k=0}^{n-1} B_{ik},
\label{eq:maskstart_J_def}
\end{equation}
and let $k^\star$ be the first step where $B_{ik^\star}=1$.
When $x_i=m$, the update at step $k$ draws $B_{ik}\sim\mathrm{Bernoulli}(p_{ik}(x))$ and, if $B_{ik}=1$,
sets $x'_{i}\sim q_{i,k}(\cdot\mid x)$ (thus $x'_i\in \mathcal{V}$); otherwise it keeps $x'_i=m$.
Hence $x^{(k^\star)}_i=m$ and
after the update we have $x^{(k^\star+1)}_i \in \mathcal{V}$.
Then Assumption~\ref{assum:unmask_only} implies that for all later steps $\ell>k^\star$,
$p_{i\ell}(x^{(\ell)})=0$ and thus $B_{i\ell}=0$.
Thus, only one edit can occur, so $J_i=1$.
\end{proof}

Proposition~\ref{prop:maskstart_single_edit} has the following implications on the variance.
Since $J_i\in\{0,1\}$, it is a Bernoulli random variable and
\begin{equation}
\mathrm{Var}(J_i)\le \tfrac14.
\end{equation}
Moreover, many mask-start implementations explicitly force any remaining \texttt{[MASK]} tokens to be
resolved by the terminal time (e.g., by setting $p_{i,n-1}(x)\equiv 1$ for $x_i=m$ at the final step),
in which case $J_i=1$ deterministically and $\mathrm{Var}(J_i)=0$.

\paragraph{Interpretation (why SHS is mainly needed for uniform-noise start).}
Under mask-start, the state explicitly reveals whether a position has been generated via the sentinel
$m=\texttt{[MASK]}$, and the unmask-only constraint \eqref{eq:maskstart_absorb_condition} enforces a
single-edit trajectory per position.
In contrast, uniform-noise start typically requires multiple self-correction edits per position,
where sampler randomness in the number and timing of edits becomes substantial (Section~\ref{prelim}),
motivating SHS as a minimal-variance event scheduler.

\section{Experimental Details and Additional Results}
\label{appen:details}

\subsection{Setup}
Checkpoints are \texttt{kuleshov-group/udlm-lm1b} (UDLM, trained on LM1B), \texttt{dvruette/gidd-base-p\_unif-0.2} (GIDD), \texttt{dvruette/gidd-unif-3b} (3B), and the 1.3B DFM checkpoint released with FS-DFM \citep{monsefi2025fs}.
Generation length follows each model's own configuration: $128$ tokens (UDLM), $512$ (GIDD), $256$ (3B), $1024$ (DFM); generation is unconditional.
Gen.\ PPL is $\exp\!\big(\sum_{m,t}-\log p_{\mathrm{GPT2}}(x_{m,t}\mid x_{m,<t})\,/\sum_m T_m\big)$ over all generated tokens, scored on the decoded text re-tokenized by GPT-2 (for DFM, on the generated token ids directly); empty samples are excluded.
Its absolute value depends on the evaluator, the generation length and the training corpus, so Gen.\ PPL is compared only between the two samplers within a model.
GIDD samples are scored without the self-correction pass of its released pipeline.
Entropy is the unigram entropy (nats) of the generated tokens pooled over the $1024$ samples; for DFM it is the per-sequence entropy in bits, averaged over samples.
Distinct-$n$ is the fraction of unique $n$-grams over whitespace tokens pooled over samples; Self-BLEU is BLEU-4 of $500$ randomly drawn samples against $50$ random references with a fixed RNG.
Special tokens are retained or stripped according to each model's own decoding, so diversity is compared only between the two samplers within a model.

\subsection{Additional Gen.\ PPL results}
\label{appen:baselines}
On UDLM, SHS at NFE${=}128$ ($175.7$) is below the standard sampler at NFE${=}1024$ ($181.6$), so the gain on this model exceeds that of an eightfold increase in NFE.
SHS decides \emph{when} a position is edited and never \emph{what} it is replaced with.
Adaptive scheduling also decides when to edit, but keeps only the ranking of the change probabilities and discards their values, which SHS matches in expectation; it raises entropy by about $1.8$ nats and Gen.\ PPL by $2$--$5\times$ (Table~\ref{tab:app_adaptive}), producing high-entropy text rather than sharper text.
Nucleus sampling and temperature reshape the destination distribution instead: both are applied to the destination distribution after the edit decision, so only the replacement token changes, and both compose with SHS at every setting (Table~\ref{tab:app_compose}).
At matched entropy, SHS alone lies on the temperature frontier (Gen.\ PPL $84.8$ at entropy $7.234$, against $\approx85.4$ for temperature interpolated to the same entropy; bootstrap difference $-0.7$, 95\% CI $[-3.4,+2.6]$) without modifying $q_{i,k}$, and adding SHS on top of temperature lowers Gen.\ PPL at every $\tau$.

\begin{table}[h]\centering\small
\caption{Same-axis comparison on the 3B model (single seed). Entropy / Gen.\ PPL.}
\label{tab:app_adaptive}
\begin{tabular}{rccc}
\toprule
NFE & Standard & Adaptive (top-1) & SHS \\
\midrule
4   & 6.758 / 936.3 & 8.248 / 1995.0 & 6.634 / 869.5 \\
8   & 6.660 / 443.4 & 8.384 / 1483.8 & 6.536 / 417.2 \\
16  & 6.612 / 221.4 & 8.459 / 1094.8 & 6.554 / 218.1 \\
32  & 6.649 / 150.2 & 8.454 / 783.7  & 6.588 / 142.0 \\
64  & 6.654 / 113.7 & 8.421 / 591.9  & 6.619 / 113.0 \\
128 & 6.649 / 104.9 & 8.483 / 574.9  & 6.618 / 99.6  \\
\bottomrule
\end{tabular}
\end{table}

\begin{table}[h]\centering\small
\caption{Composition with nucleus sampling and temperature at NFE${=}128$ (mean of 3 seeds). Entropy / Gen.\ PPL.}
\label{tab:app_compose}
\begin{tabular}{lccc}
\toprule
Model & Standard & top-$p$ $0.9$ & top-$p$ $0.9$ + SHS \\
\midrule
UDLM & 6.884 / 190.65 & 6.866 / 187.45 & 6.817 / 174.55 \\
GIDD & 7.350 / 111.63 & 7.299 / 97.56  & 7.206 / 78.23  \\
\midrule
GIDD, $\tau$ & & Temperature & Temperature + SHS \\
\midrule
1.0   & & 7.350 / 111.63 & 7.234 / 84.83 \\
0.99  & & 7.301 / 99.89  & 7.182 / 76.09 \\
0.975 & & 7.218 / 81.91  & 7.118 / 65.02 \\
0.95  & & 7.081 / 63.34  & 7.002 / 51.72 \\
0.90  & & 6.853 / 40.57  & 6.781 / 34.86 \\
\bottomrule
\end{tabular}
\end{table}

\subsection{Diversity and fidelity}
\label{appen:diversity}
The standard sampler is more diverse than real text and SHS moves toward real text (Table~\ref{tab:app_realdata}); the reference is OpenWebText cut into blocks matching the average word count of the generated texts, and all metrics use whitespace tokens.
For NFE $\ge16$, $88$--$90\%$ of the Distinct-1 reduction is token types absent from the reference corpus; tokens common in the corpus contribute under $2\%$.
Under the generating model's own per-token ELBO (Monte Carlo over diffusion time, $128$ draws for UDLM and $64$ for GIDD, seeds 1--3), SHS samples score lower in every cell: UDLM at NFE 16/32/128/512, $4.676\to4.580$, $4.459\to4.396$, $4.357\to4.265$, $4.315\to4.213$; GIDD at NFE 16/128/512, $5.151\to4.946$, $4.394\to4.195$, $4.337\to4.049$.
MAUVE (GPT2-large features, $102$ buckets, $1024$ samples per seed against $5000$ reference texts: LM1B validation for UDLM, OpenWebText for GIDD) rises on GIDD from $0.237$ to $0.303$ (NFE 32), $0.381$ to $0.536$ (128) and $0.418$ to $0.559$ (512); on UDLM both samplers stay below $0.15$, too low to separate.

\begin{table}[h]\centering\small
\caption{GIDD lexical diversity against length-matched real text (std / SHS).}
\label{tab:app_realdata}
\begin{tabular}{lcccc}
\toprule
 & D-1 & D-2 & D-3 & Self-BLEU-2 \\
\midrule
Real data & 0.148 & 0.626 & 0.903 & 0.062 \\
NFE 16  & 0.180 / 0.171 & 0.702 / 0.688 & 0.958 / 0.955 & 0.057 / 0.059 \\
NFE 128 & 0.168 / 0.153 & 0.677 / 0.656 & 0.946 / 0.941 & 0.053 / 0.056 \\
NFE 512 & 0.169 / 0.153 & 0.678 / 0.654 & 0.946 / 0.939 & 0.052 / 0.054 \\
\bottomrule
\end{tabular}
\end{table}

\subsection{Verification under state-dependent sampling and phase ablation}
\label{appen:closedloop}
Measured on the masses each position actually accumulated under state-dependent sampling, the invariant $J_i\in\{\lfloor S_i^{\mathrm{tot}}\rfloor,\lceil S_i^{\mathrm{tot}}\rceil\}$ holds at all $7{,}864{,}320$ logged positions, the conditional variance repeats with the unit period in $S_i^{\mathrm{tot}}$ that $f(1-f)$ predicts (measured period $0.966$--$1.004$, NFE $8$--$1024$), and its magnitude is $2.8\times$ (NFE 4) to $48.8\times$ (NFE 1024) below the standard sampler's.
Table~\ref{tab:app_phase} ablates the phase on UDLM at NFE${=}128$ with 3 seeds and paired tests: the shared setting draws one $\theta$ per sequence for all positions, the redrawn setting replaces a position's $\theta$ after each of its edits.
The shared and per-position settings agree in $\mathbb{E}[S^{\mathrm{tot}}]$ to within $0.16\%$, so the $39$-point gap comes from the coupling across positions rather than from the edit budget.
The law of $\theta$ and the unit threshold spacing are not free choices: Propositions~\ref{prop:unbiased_main}--\ref{prop:variance_main} force both.

\begin{table}[h]\centering\small
\caption{Phase ablation on UDLM, NFE${=}128$, 3 seeds.}
\label{tab:app_phase}
\begin{tabular}{lccc}
\toprule
Phase $\theta$ & Gen.\ PPL & paired $p$ vs.\ per-position & $\mathbb{E}[S^{\mathrm{tot}}]$ \\
\midrule
Independent per position (SHS) & 174.248 & --     & 5.0789 \\
Shared across positions        & 213.477 & 0.0039 & 5.0868 \\
Redrawn at each unit interval  & 173.711 & 0.799  & --     \\
\bottomrule
\end{tabular}
\end{table}

\subsection{Count deviation stratified by cumulative mass}
\label{appen:stratified}
Token NLL in Figure~\ref{fig:count_deviation_quality} is the GPT2-large NLL mapped to each model position by character offsets.
Stratifying by $S_i^{\mathrm{tot}}$ removes $81$--$84\%$ of the raw excess NLL of positions with $\Delta J_i\neq0$ over positions with $\Delta J_i=0$ ($0.287\to0.055$ at NFE 16, $0.127\to0.020$ at NFE 128) and leaves the under-edit penalty significant in all 12 strata; the over-edit side is significantly negative in 10 of 12 (Table~\ref{tab:stratified_excess_nll}).

\begin{table}[h]\centering\small
\caption{Excess token NLL (nats) relative to $\Delta J_i=0$ within each $S_i^{\mathrm{tot}}$ stratum (width 1), UDLM standard sampler, seeds 1--5. Brackets: 95\% CI, sentence-level cluster bootstrap.}
\label{tab:stratified_excess_nll}
\resizebox{\linewidth}{!}{\begin{tabular}{c cc cc}
\toprule
& \multicolumn{2}{c}{NFE $=16$} & \multicolumn{2}{c}{NFE $=128$} \\
\cmidrule(lr){2-3} \cmidrule(lr){4-5}
$S_i^{\mathrm{tot}}$ & under-edit ($\Delta J_i\le-1$) & over-edit ($\Delta J_i\ge+1$) & under-edit ($\Delta J_i\le-1$) & over-edit ($\Delta J_i\ge+1$) \\
\midrule
$[1,2)$ & $+1.991$ \scriptsize{$[+1.910, +2.070]$} & $-0.259$ \scriptsize{$[-0.312, -0.208]$} & $+0.716$ \scriptsize{$[+0.509, +0.927]$} & $-0.350$ \scriptsize{$[-0.466, -0.238]$} \\
$[2,3)$ & $+0.946$ \scriptsize{$[+0.899, +0.995]$} & $-0.445$ \scriptsize{$[-0.478, -0.414]$} & $+0.608$ \scriptsize{$[+0.512, +0.708]$} & $-0.366$ \scriptsize{$[-0.442, -0.293]$} \\
$[3,4)$ & $+0.397$ \scriptsize{$[+0.359, +0.436]$} & $-0.155$ \scriptsize{$[-0.189, -0.119]$} & $+0.543$ \scriptsize{$[+0.481, +0.602]$} & $-0.286$ \scriptsize{$[-0.337, -0.237]$} \\
$[4,5)$ & $+0.281$ \scriptsize{$[+0.236, +0.324]$} & $-0.051$ \scriptsize{$[-0.099, -0.003]$} & $+0.307$ \scriptsize{$[+0.261, +0.353]$} & $-0.182$ \scriptsize{$[-0.221, -0.144]$} \\
$[5,6)$ & $+0.375$ \scriptsize{$[+0.292, +0.457]$} & $-0.110$ \scriptsize{$[-0.206, -0.003]$} & $+0.177$ \scriptsize{$[+0.132, +0.220]$} & $-0.097$ \scriptsize{$[-0.141, -0.053]$} \\
$[6,\infty)$ & $+0.248$ \scriptsize{$[+0.135, +0.358]$} & $+0.061$ \scriptsize{$[-0.122, +0.243]$} & $+0.110$ \scriptsize{$[+0.066, +0.157]$} & $-0.045$ \scriptsize{$[-0.094, +0.004]$} \\
\bottomrule
\end{tabular}}
\end{table}

\section{Ablation: DFM Change-Mass Discretization}
\label{appen:dfm_discretization}

DFM samplers require converting a local model output $\lambda_{ik}$ into a valid per-step change mass $p_{ik}\in[0,1]$.
We ablate the choice of discretization, comparing two forms:
\[
p_{ik}^{\mathrm{Euler}}=\min(h\lambda_{ik},1),
\qquad
p_{ik}^{\exp}=1-\exp(-h\lambda_{ik}).
\]
The exponential form is the frozen-rate CTMC event probability. The Euler form is the first-order $\tau$-leaping mass used in our main DFM experiments. Since $1-\exp(-x)\le x$, the exponential form yields a smaller effective cumulative mass for positive $x$ relative to the model's intended flow $\lambda h$.

\begin{table}[h]
\centering
\caption{
DFM generation quality (Gen.~PPL, lower is better) under the two change-mass discretizations, standard sampler, $320$ samples, single seed; $\pm$ is the standard deviation over samples. The percentage column reports the relative reduction of $p^{\mathrm{Euler}}$ over $p^{\exp}$.
}
\vspace{0.5em}
\label{tab:dfm_discretization}
\begin{tabular}{rccc}
\toprule
NFE & $p^{\exp} = 1-\exp(-h\lambda)$ & $p^{\mathrm{Euler}} = \min(h\lambda,1)$ & $\Delta$ \\
\midrule
   2 & $1150.5 \pm 157.9$ & $759.2 \pm 94.5$ & $-34.0\%$ \\
   4 & $372.7 \pm 55.8$   & $286.5 \pm 39.5$ & $-23.1\%$ \\
   8 & $158.4 \pm 21.8$   & $141.0 \pm 17.3$ & $-11.0\%$ \\
  16 & $86.3 \pm 12.1$    & $83.9 \pm 10.3$  & $-2.8\%$  \\
  32 & $52.3 \pm 10.3$    & $53.4 \pm 7.3$   & $+2.1\%$  \\
  64 & $29.5 \pm 10.3$    & $28.7 \pm 10.6$  & $-2.7\%$  \\
\bottomrule
\end{tabular}
\end{table}

The Gen.~PPL gap in Table~\ref{tab:dfm_discretization} at small NFE motivates using the Euler change mass for the DFM experiments reported in the main text.

\section{Broader Impact}

SHS is a training-free, hyperparameter-free inference principle that can be dropped into any CTMC/DTMC sampler admitting a stay-vs.-replace decomposition, at virtually no computational cost and with no extra forward passes. By achieving the minimum possible jump-count variance ($\leq 1/4$) for any fixed cumulative mass, it directly addresses a sampler-induced failure mode---under- and over-editing---identified empirically in our diagnostics. The consistent improvements in generative perplexity across diverse uniform-noise backbones spanning $\sim$110M to $\sim$3B parameters and covering both diffusion- and flow-based families (UDLM, GIDD, DFM, and a 3B uniform diffusion model) indicate that the benefit stems from the sampler rather than from model-specific artifacts.

Since SHS preserves the learned destination kernels and only changes the temporal coupling of stay-vs.-replace events, it does not alter what the underlying model is willing to produce. Its societal risk profile is therefore inherited from the underlying generative models, and the same considerations around misuse of any text generator apply.

\section{Hardware Specifications}
All experiments were conducted on a system with 256 CPU cores, 4 Nvidia RTX L40S (48GB) GPUs,
and 1024 GB of RAM. In total, the experiments, evaluations, analyses, and method development
required approximately 500 GPU hours.

\section{Qualitative Examples}
\label{appen:examples}
Samples are matched by model, NFE and seed, taken in stored order (indices 0--2) without selection, cut at 60 words for display, with per-sample Gen.\ PPL; in the UDLM pairs, tokens are colored by $\Delta J_i$ (red under-edited, blue over-edited, darker for larger $|\Delta J_i|$).

\paragraph{UDLM, NFE${=}16$, seed 1.}\leavevmode\par
\noindent\textbf{Standard, idx 0} (Gen.\ PPL 551.0)\par
{\small \textcolor{shsunder!50}{[CLS]} \textcolor{shsover!75}{waterford} \textcolor{shsover!75}{,} \textcolor{shsunder!75}{mic}h \textcolor{shsover!50}{.} \textcolor{shsover!75}{,} \textcolor{shsunder!75}{june} \textcolor{shsover!50}{21} \textcolor{shsunder!50}{/} \textcolor{shsover!50}{pr}\textcolor{shsover!50}{ne}\textcolor{shsover!50}{ws}wire / \textcolor{shsover!50}{-} \textcolor{shsover!75}{-} \textcolor{shsover!50}{never}land has \textcolor{shsunder!75}{returned} \textcolor{shsunder!75}{to} \textcolor{shsunder!100}{pu}rne to be \textcolor{shsover!50}{the} \textcolor{shsover!50}{first} of the u \textcolor{shsunder!50}{.} s \textcolor{shsunder!50}{.} ships \textcolor{shsunder!75}{elizabeth} \textcolor{shsunder!75}{europe} \textcolor{shsover!50}{.} \textcolor{shsover!50}{along} \textcolor{shsunder!75}{to} \textcolor{shsunder!50}{bring} \textcolor{shsunder!50}{a} \textcolor{shsunder!75}{lot} more \textcolor{shsover!75}{food} \textcolor{shsover!100}{and} \textcolor{shsunder!50}{visitors} . \textcolor{shsover!50}{ve}\textcolor{shsover!50}{lem}\textcolor{shsunder!50}{outh} had \$ \textcolor{shsover!100}{1} \textcolor{shsunder!100}{fighter} \textcolor{shsunder!75}{arena}\textcolor{shsunder!75}{ma} \textcolor{shsunder!75}{designed} \textcolor{shsover!50}{at} \textcolor{shsunder!50}{port} 1 mount st \textcolor{shsover!100}{.} \textcolor{shsover!75}{,} \textcolor{shsover!50}{according} to \textcolor{shsover!75}{su}\textcolor{shsunder!50}{sb}\textcolor{shsover!50}{ast}\textcolor{shsunder!50}{le} \textcolor{shsover!50}{.} [\ldots]}\par\medskip
\noindent\textbf{Standard, idx 1} (Gen.\ PPL 96.5)\par
{\small \textcolor{shsover!50}{[CLS]} \textcolor{shsover!50}{in} other \textcolor{shsunder!75}{words} , \textcolor{shsover!75}{says} \textcolor{shsover!50}{ren}\textcolor{shsover!100}{di} \textcolor{shsunder!75}{.} \textcolor{shsover!50}{so} \textcolor{shsunder!100}{credit} \textcolor{shsunder!50}{in} the \textcolor{shsunder!75}{world} \textcolor{shsover!50}{leader} \textcolor{shsover!75}{has} the \textcolor{shsover!50}{resources} \textcolor{shsunder!75}{to} \textcolor{shsunder!50}{find} \textcolor{shsover!50}{urgently} needed \textcolor{shsover!50}{relief} \textcolor{shsover!100}{and} \textcolor{shsunder!75}{refer}\textcolor{shsunder!50}{ral}\textcolor{shsover!50}{s} to other \textcolor{shsunder!50}{countries} and \textcolor{shsover!50}{economies} . \textcolor{shsunder!50}{and} since haiti \textcolor{shsover!75}{was} \textcolor{shsover!50}{born} \textcolor{shsover!75}{in} latin \textcolor{shsunder!75}{america} , \textcolor{shsunder!75}{which} the united states \textcolor{shsunder!75}{signed} \textcolor{shsunder!50}{upon} by \textcolor{shsunder!50}{france} , \textcolor{shsunder!50}{its} \textcolor{shsover!100}{position} \textcolor{shsover!75}{-} \textcolor{shsover!75}{-} \textcolor{shsover!75}{as} \textcolor{shsover!50}{most} influential \textcolor{shsunder!50}{nation} \textcolor{shsover!50}{in} \textcolor{shsover!50}{western} nations - [\ldots]}\par\medskip
\noindent\textbf{Standard, idx 2} (Gen.\ PPL 231.0)\par
{\small \textcolor{shsover!50}{[CLS]} \textcolor{shsunder!75}{it} \textcolor{shsover!50}{was} \textcolor{shsover!75}{last} \textcolor{shsover!50}{updated} \textcolor{shsunder!75}{at} \textcolor{shsunder!50}{01} \textcolor{shsover!50}{.} \textcolor{shsunder!50}{39} gm\textcolor{shsover!50}{t} \textcolor{shsover!75}{on} \textcolor{shsover!75}{saturday} 24 february \textcolor{shsunder!50}{2009} \textcolor{shsover!50}{.} editorial \textcolor{shsover!75}{associates} \textcolor{shsunder!75}{.} central \textcolor{shsover!50}{london} with \textcolor{shsover!50}{a} \textcolor{shsunder!50}{francis} per \textcolor{shsover!50}{annum} bu \textcolor{shsunder!50}{.} \textcolor{shsunder!50}{.} \textcolor{shsover!50}{.} . \pounds{}2500\textcolor{shsover!50}{0} \textcolor{shsover!50}{-} \textcolor{shsover!75}{\pounds{}2}\textcolor{shsover!75}{400}\textcolor{shsover!50}{0} \textcolor{shsover!50}{per} \textcolor{shsover!50}{annum} \textcolor{shsover!50}{+} benefits \textcolor{shsover!75}{.} \textcolor{shsover!50}{the} \textcolor{shsunder!50}{post} \textcolor{shsunder!50}{office} \textcolor{shsunder!50}{will} \textcolor{shsover!75}{provide} \textcolor{shsunder!50}{.} \textcolor{shsunder!50}{you} will \textcolor{shsover!50}{with} \textcolor{shsover!75}{direct} and \textcolor{shsover!75}{marketing} re\textcolor{shsunder!75}{mun} \textcolor{shsunder!50}{ba} . . . \textcolor{shsunder!50}{.} \pounds{}\textcolor{shsunder!50}{60} \textcolor{shsover!50}{,} \textcolor{shsover!50}{000} [\ldots]}\par\medskip
\noindent\textbf{SHS, idx 0} (Gen.\ PPL 403.0)\par
{\small [CLS] \textcolor{shsover!50}{one} last \textcolor{shsover!50}{of} a \textcolor{shsunder!50}{kind} sun - tan post - exhibition \textcolor{shsover!50}{and} policy conference \textcolor{shsover!50}{run} by governmentolov \textcolor{shsover!50}{,} ms blear vsa and more often will be the leaders of an al qaeda site \textcolor{shsover!50}{,} tehran . \textcolor{shsover!50}{put} on madak ' s performance \textcolor{shsover!50}{as} \textcolor{shsover!50}{the} director of pakistan ' s \textcolor{shsover!50}{''} gr\textcolor{shsover!50}{ei}fi now '' book watch network , a ma\textcolor{shsover!50}{hdi} [\ldots]}\par\medskip
\noindent\textbf{SHS, idx 1} (Gen.\ PPL 216.1)\par
{\small [CLS] irrelevant . yeah \textcolor{shsover!50}{u} . \textcolor{shsover!50}{s} . follow or \textcolor{shsunder!50}{maybe} we are \textcolor{shsover!50}{going} to . \textcolor{shsover!50}{to} getting auto\textcolor{shsover!50}{s} only \textcolor{shsover!50}{go} on making it for a \textcolor{shsunder!50}{budget} market because they \textcolor{shsover!50}{do} . the american healthcare bill more than adequate \textcolor{shsover!50}{-} \textcolor{shsover!50}{it} provides no new help for these people who have that \textcolor{shsover!50}{problem} . more anxious \textcolor{shsover!50}{and} \textcolor{shsover!50}{bothered} for the \textcolor{shsover!50}{sake} [\ldots]}\par\medskip
\noindent\textbf{SHS, idx 2} (Gen.\ PPL 168.7)\par
{\small [CLS] a whiz on '' \textcolor{shsover!50}{the} only trouble thinking americans '' when it \textcolor{shsover!50}{comes} off \textcolor{shsover!50}{the} front \textcolor{shsover!50}{.} back in america \textcolor{shsover!50}{'} s ' 70s that turn of the head down \textcolor{shsover!50}{14th} street \textcolor{shsover!50}{where} cash\textcolor{shsunder!50}{ead} ' s death \textcolor{shsover!50}{was} touched was a beginning \textcolor{shsover!50}{for} both \textcolor{shsover!50}{my} o3 \textcolor{shsover!50}{and} \textcolor{shsover!50}{our} most important lunch audience . \textcolor{shsover!50}{but} , \textcolor{shsover!50}{the} \textcolor{shsover!50}{''} part \textcolor{shsover!50}{dr} [\ldots]}\par\medskip

\paragraph{UDLM, NFE${=}128$, seed 1.}\leavevmode\par
\noindent\textbf{Standard, idx 0} (Gen.\ PPL 176.5)\par
{\small \textcolor{shsunder!100}{[CLS]} \textcolor{shsover!75}{she} \textcolor{shsover!50}{was} a \textcolor{shsover!100}{successful} student and \textcolor{shsunder!50}{advocate} \textcolor{shsover!50}{for} \textcolor{shsunder!75}{o}cc \textcolor{shsover!75}{and} \textcolor{shsunder!75}{its} \textcolor{shsunder!100}{mission} \textcolor{shsunder!50}{including} \textcolor{shsunder!50}{women} \textcolor{shsover!100}{and} \textcolor{shsover!100}{retire}\textcolor{shsunder!50}{es} \textcolor{shsover!75}{and} workers \textcolor{shsunder!75}{of} \textcolor{shsunder!75}{the} \textcolor{shsunder!50}{city} from \textcolor{shsover!50}{1996} \textcolor{shsover!50}{to} 2004 \textcolor{shsunder!50}{\textquoteright{}} \textcolor{shsover!50}{it} \textcolor{shsunder!50}{entrepreneurship} \textcolor{shsunder!75}{'} \textcolor{shsunder!50}{for} the \textcolor{shsunder!75}{local} \textcolor{shsunder!100}{business} community \textcolor{shsover!75}{.} \textcolor{shsover!75}{the} \textcolor{shsunder!75}{web} \textcolor{shsover!100}{site} \textcolor{shsover!50}{began} \textcolor{shsunder!100}{operating} \textcolor{shsunder!75}{today} \textcolor{shsover!50}{shortly} \textcolor{shsover!100}{before} \textcolor{shsunder!100}{visitors} . \textcolor{shsunder!75}{this} \textcolor{shsunder!50}{weekend} , \textcolor{shsunder!100}{q}\textcolor{shsunder!50}{e} \textcolor{shsover!100}{1} \textcolor{shsunder!50}{.} \textcolor{shsunder!100}{com} \textcolor{shsover!50}{decided} \textcolor{shsunder!50}{to} \textcolor{shsover!50}{dispatch} \textcolor{shsunder!100}{writer} \textcolor{shsunder!100}{,} \textcolor{shsunder!100}{ezra} \textcolor{shsunder!100}{es}\textcolor{shsunder!75}{zel}\textcolor{shsover!75}{an} [\ldots]}\par\medskip
\noindent\textbf{Standard, idx 1} (Gen.\ PPL 353.1)\par
{\small \textcolor{shsover!50}{[CLS]} \textcolor{shsunder!75}{circa} \textcolor{shsunder!50}{\pounds{}3}3 \textcolor{shsover!50}{,} \textcolor{shsover!100}{000} \textcolor{shsunder!100}{.} \textcolor{shsover!75}{to} \textcolor{shsover!50}{plan} and \textcolor{shsover!50}{work} in \textcolor{shsunder!50}{an} \textcolor{shsunder!50}{online} video \textcolor{shsover!100}{game} \textcolor{shsunder!50}{.} \textcolor{shsunder!50}{stem} \textcolor{shsunder!100}{fisher} \textcolor{shsover!75}{.} \textcolor{shsover!50}{a} \textcolor{shsunder!50}{doctor} \textcolor{shsover!50}{offers} \textcolor{shsunder!75}{geo} \textcolor{shsunder!75}{stories} \textcolor{shsover!50}{(} \textcolor{shsunder!75}{you} \textcolor{shsunder!100}{have} \textcolor{shsover!50}{a} \textcolor{shsunder!100}{weekly} \textcolor{shsunder!50}{o} . . \textcolor{shsunder!50}{.} \textcolor{shsover!100}{.} \textcolor{shsunder!50}{\pounds{}3}\textcolor{shsover!75}{500}\textcolor{shsover!100}{0} \textcolor{shsover!75}{-} \textcolor{shsunder!50}{\pounds{}}42\textcolor{shsover!75}{9}\textcolor{shsunder!75}{9} \textcolor{shsover!100}{per} \textcolor{shsover!75}{annum} + \textcolor{shsunder!75}{performance} allowance \textcolor{shsover!75}{or} on \textcolor{shsunder!75}{clearance} \textcolor{shsunder!50}{level} \textcolor{shsover!100}{.} - \textcolor{shsover!75}{-} \textcolor{shsover!75}{higher} \textcolor{shsover!75}{education} \textcolor{shsunder!50}{education} \textcolor{shsunder!75}{:} \textcolor{shsover!75}{to} work with \textcolor{shsunder!50}{a} \textcolor{shsunder!75}{u} \textcolor{shsunder!100}{.} \textcolor{shsunder!50}{.} [\ldots]}\par\medskip
\noindent\textbf{Standard, idx 2} (Gen.\ PPL 275.9)\par
{\small \textcolor{shsunder!50}{[CLS]} \textcolor{shsunder!100}{existing} \textcolor{shsover!75}{opt}af\textcolor{shsunder!50}{ar} \textcolor{shsunder!50}{base} \textcolor{shsunder!50}{price} requirements \textcolor{shsunder!50}{may} even \textcolor{shsover!75}{be} reflected \textcolor{shsover!50}{in} \textcolor{shsover!50}{comparable} \textcolor{shsunder!75}{product} \textcolor{shsunder!50}{to} \textcolor{shsover!50}{improve} \textcolor{shsover!75}{conditions} \textcolor{shsover!75}{or} \textcolor{shsunder!50}{demand} \textcolor{shsover!75}{;} \textcolor{shsover!50}{the} \textcolor{shsover!75}{consumption} \textcolor{shsover!75}{of} \textcolor{shsover!50}{d}\textcolor{shsover!100}{vr}\textcolor{shsover!75}{s} \textcolor{shsover!50}{and} \textcolor{shsunder!100}{principally} \textcolor{shsunder!75}{;} \textcolor{shsunder!50}{the} \textcolor{shsover!75}{overall} \textcolor{shsunder!75}{release} \textcolor{shsunder!100}{to} \textcolor{shsover!50}{its} \textcolor{shsunder!50}{product} \textcolor{shsover!100}{back}log to \textcolor{shsover!100}{attract} new sales programs \textcolor{shsover!50}{and} expanding services \textcolor{shsover!75}{that} will \textcolor{shsover!100}{lower} \textcolor{shsunder!100}{gross} \textcolor{shsover!100}{revenues} \textcolor{shsunder!50}{in} \textcolor{shsunder!50}{other} \textcolor{shsover!100}{measures} \textcolor{shsover!50}{by} \textcolor{shsunder!50}{a} net \textcolor{shsunder!50}{one} \textcolor{shsover!50}{percent} \textcolor{shsover!50}{:} \textcolor{shsunder!75}{a} reduction in \textcolor{shsover!75}{service} \textcolor{shsover!50}{and} [\ldots]}\par\medskip
\noindent\textbf{SHS, idx 0} (Gen.\ PPL 272.1)\par
{\small [CLS] \textcolor{shsover!50}{we} really \textcolor{shsover!50}{had} a continuing \textcolor{shsover!50}{story} during the first full week \textcolor{shsover!50}{of} 2008 \textcolor{shsover!50}{,} \textcolor{shsover!50}{but} these \textcolor{shsover!50}{builders} \textcolor{shsover!50}{and} developers entered this inevitable new phase of disfiguring the stock market , '' said wgc \textcolor{shsover!50}{analyst} \textcolor{shsover!50}{john} dragon\textcolor{shsover!50}{g} \textcolor{shsover!50}{,} \textcolor{shsover!50}{looking} \textcolor{shsover!50}{back} now that lori\textcolor{shsover!50}{llar}\textcolor{shsover!50}{d} \textcolor{shsover!50}{has} \textcolor{shsunder!50}{closed} 11 '' \textcolor{shsover!50}{dealers}hips \textcolor{shsover!50}{''} \textcolor{shsover!50}{at} \textcolor{shsover!50}{an} 18 - member equipment \textcolor{shsover!50}{shop} chain ' \textcolor{shsover!50}{s} location [\ldots]}\par\medskip
\noindent\textbf{SHS, idx 1} (Gen.\ PPL 199.6)\par
{\small \textcolor{shsover!50}{[CLS]} it was \textcolor{shsover!50}{last} \textcolor{shsover!50}{updated} at \textcolor{shsover!50}{11} . 24 bst on thursday \textcolor{shsover!50}{1} \textcolor{shsover!50}{may} 2009 \textcolor{shsover!50}{.} liberty is \textcolor{shsover!50}{an} early \textcolor{shsover!50}{education} center \textcolor{shsover!50}{based} on fantastic or lig . . \textcolor{shsover!50}{.} . \textcolor{shsover!50}{\pounds{}3}5 , 000 to \textcolor{shsover!50}{\pounds{}3}3 \textcolor{shsover!50}{,} \textcolor{shsover!50}{79}\textcolor{shsover!50}{9} per annum + 15 \% contribution to \pounds{}2\textcolor{shsover!50}{0} \textcolor{shsover!50}{,} 000 \pounds{}37 \textcolor{shsover!50}{,} \textcolor{shsover!50}{93}3 to \textcolor{shsover!50}{\pounds{}}\textcolor{shsover!50}{44} \textcolor{shsover!50}{,} 833 pa \textcolor{shsover!50}{.} base - salary [\ldots]}\par\medskip
\noindent\textbf{SHS, idx 2} (Gen.\ PPL 165.4)\par
{\small [CLS] spectra \textcolor{shsunder!50}{channel}au ' s recent \textcolor{shsover!50}{fiscal} \textcolor{shsover!50}{year} \textcolor{shsover!50}{2009} gold production when \textcolor{shsover!50}{it} attributes \textcolor{shsover!50}{to} the \textcolor{shsover!50}{expansion} in \textcolor{shsover!50}{the} \textcolor{shsover!50}{dollar} market . \textcolor{shsover!50}{the} quarter used to offset a slight increase in additions \textcolor{shsover!50}{by} customers \textcolor{shsover!50}{and} the \textcolor{shsunder!50}{rate} of redflow\textcolor{shsover!50}{s} and in both \textcolor{shsover!50}{the} \textcolor{shsover!50}{prices} of \textcolor{shsover!50}{low} - \textcolor{shsover!50}{yielding} \textcolor{shsover!50}{cu}\textcolor{shsover!50}{rre}ncies . over \textcolor{shsover!50}{the} winter \textcolor{shsover!50}{the} company \textcolor{shsover!50}{expects} \textcolor{shsover!50}{volumes} to \textcolor{shsover!50}{be} \textcolor{shsover!50}{higher} [\ldots]}\par\medskip

\paragraph{3B model, NFE${=}128$, seed 0 (corpus Gen.\ PPL: Standard 104.9, SHS 99.6).}\leavevmode\par
\noindent\textbf{Standard, idx 0} (Gen.\ PPL 88.0)\par
{\small of content confirmed expanded the club , and the influx of interest and information requests has caused a significant content shortage on the club 's website . With even talent basic , the .COM is forced to develop an array of other formats for the community . Marko notes , '' We are grateful for your continued support , but [\ldots]}\par\medskip
\noindent\textbf{Standard, idx 1} (Gen.\ PPL 200.1)\par
{\small stationery x . 5 / wide / 1 x . 5 / high / 4 x .e / warmth Noonan 2 x . 5 / high . 5 x . 75 / depth / 2 x . 5 / deep / Important support us please click / If would you like to reply us ? / Click here . [\ldots]}\par\medskip
\noindent\textbf{Standard, idx 2} (Gen.\ PPL 190.3)\par
{\small joy comes out of I , and I look forward to returning for more .'' soon / importance i was really charged by your ,S and your day what 's calling the common genre , and like the direction of Jesus songs think it is good bit mysteries , but I will be more watching [\ldots]}\par\medskip
\noindent\textbf{SHS, idx 0} (Gen.\ PPL 53.5)\par
{\small had never been alone before and she was nervous on nervous . ''I guess , if you want to take off , I think you can wear go with me .'' Her housemate 's tone was concerned . she didn 't have any words for herself but nevermind her face . She tried to leave the room walking slowly around [\ldots]}\par\medskip
\noindent\textbf{SHS, idx 1} (Gen.\ PPL 148.3)\par
{\small as normal on Monday night . Looking should be a real strong race tracks on the promotion crew . Week Revival has been attracting racers in and out since 198 1 is the 5 th International Trail . It is one of the no one races for Monaco major 's held in Columbus , Indianapolis , Atlanta , Wisconsin , [\ldots]}\par\medskip
\noindent\textbf{SHS, idx 2} (Gen.\ PPL 94.9)\par
{\small player that tends to develop a habit of making good bets , although a player that should adjusted Not will be too willing to protect their bluff and will get in the pot , making other players seem weak . Putting Your / -Game Action : / Now , you all know how , and we all know how . [\ldots]}\par\medskip

\paragraph{GIDD, NFE${=}512$, seed 1 (corpus Gen.\ PPL: Standard 103.4, SHS 76.6).}\leavevmode\par
\noindent\textbf{Standard, idx 0} (Gen.\ PPL 73.0)\par
{\small had the opportunity to watch Guardians embryonic family again with Dave Bautista as well as Victoria Sumpter and Gwyneth Paltrow. The actors had to keep together their rig while wearing clothes and makeup, he explained down the line. The shoot and movie music will be screened by one or a few fans. [\ldots]}\par\medskip
\noindent\textbf{Standard, idx 1} (Gen.\ PPL 229.0)\par
{\small , lotteries involving war movies, a comedy series with the Hound of Bjermart, a rock band ThreeGo between Fanon poll and O2 legends Hillfritoww, and a huge rival monthly club night, with a running variety of films, an occasional radio station and a campaigner video store. [\ldots]}\par\medskip
\noindent\textbf{Standard, idx 2} (Gen.\ PPL 250.6)\par
{\small through the anime Pursusta, if your not college drama manga. Her ticks? Because you chose to deserves to date with two fisticuffs only to make her crumpas pay. Go for: Jacket, accessories, starting trouble, experiences, if the answer is sensible/ ponytail:MamaLux, a \& interview convention [\ldots]}\par\medskip
\noindent\textbf{SHS, idx 0} (Gen.\ PPL 76.6)\par
{\small Others arriving with a special ability. Mr. Reynoso underscored a ''thin line between the Marines students with each lousy group of Marines and Army outposts.'' Photograph by Devin Lee. ''He had a bar band hit, made up of simply lined up in with him, whatever the music was,'' Hamere said. [\ldots]}\par\medskip
\noindent\textbf{SHS, idx 1} (Gen.\ PPL 123.5)\par
{\small ''What's the best hope in the future for policies in store?'' was Microsoft's reply to the announcement late after the night that Microsoft Store would permanently slapped down Windows Phone developers. Usability normally the most regrettable will help soften the blow, and has been chosen to kick things off. [\ldots]}\par\medskip
\noindent\textbf{SHS, idx 2} (Gen.\ PPL 159.8)\par
{\small soil. Biotechnology is also a key to continue to bring farming plants, though Kussman believes other ways tofish are sustainable. (Photo of one fish with kuyi in aquaponics. (AE image, Jim Stornian and Elisena Allill, FPL) By 2018, the FPL anomalyland farmers will be self-sufficient. [\ldots]}\par\medskip

\newpage

\section*{NeurIPS Paper Checklist}

\begin{enumerate}

\item {\bf Claims}
    \item[] Question: Do the main claims made in the abstract and introduction accurately reflect the paper's contributions and scope?
    \item[] Answer: \answerYes{} 
    \item[] Justification: We have written our abstract and also the introduction that summarizes our paper’s contributions and scope.

\item {\bf Limitations}
    \item[] Question: Does the paper discuss the limitations of the work performed by the authors?
    \item[] Answer: \answerYes{} 
    \item[] Justification: Conclusion.

\item {\bf Theory assumptions and proofs}
    \item[] Question: For each theoretical result, does the paper provide the full set of assumptions and a complete (and correct) proof?
    \item[] Answer: \answerYes{} 
    \item[] Justification: Propositions 1–4 state their assumptions in Sections 3.1 and 3.2; full proofs are given in Appendices B, C, and D.

    \item {\bf Experimental result reproducibility}
    \item[] Question: Does the paper fully disclose all the information needed to reproduce the main experimental results of the paper to the extent that it affects the main claims and/or conclusions of the paper (regardless of whether the code and data are provided or not)?
    \item[] Answer: \answerYes{} 
    \item[] Justification: Section 4.1 specifies the four pretrained models (UDLM, GIDD, DFM, 3B uniform diffusion), NFE budgets, sample counts and seeds, the GPT2-large evaluator, and the DFM change-mass form; Appendix~\ref{appen:details} lists the checkpoints, generation lengths, and metric definitions. SHS is given as complete pseudocode in Algorithm~\ref{alg:shs} with no additional hyperparameters.

\item {\bf Open access to data and code}
    \item[] Question: Does the paper provide open access to the data and code, with sufficient instructions to faithfully reproduce the main experimental results, as described in supplemental material?
    \item[] Answer: \answerYes{} 
    \item[] Justification: A reference implementation of Algorithm~\ref{alg:shs} with scripts to reproduce the main results is publicly available at the URL given in the abstract.

\item {\bf Experimental setting/details}
    \item[] Question: Does the paper specify all the training and test details (e.g., data splits, hyperparameters, how they were chosen, type of optimizer) necessary to understand the results?
    \item[] Answer: \answerYes{} 
    \item[] Justification: Section 4.1 specifies the pretrained models, NFE budgets, sample counts, seeds, evaluation metrics, and the change-mass discretization (further detailed in Appendix F). SHS is training-free and hyperparameter-free.

\item {\bf Experiment statistical significance}
    \item[] Question: Does the paper report error bars suitably and correctly defined or other appropriate information about the statistical significance of the experiments?
    \item[] Answer: \answerYes{} 
    \item[] Justification: Table 1 reports mean ± population standard deviation over 5 seeds for UDLM and GIDD and 3 seeds for DFM.
    Only 3B model uses a single seed as noted in Section 4.1 due to the high computational cost.

\item {\bf Experiments compute resources}
    \item[] Question: For each experiment, does the paper provide sufficient information on the computer resources (type of compute workers, memory, time of execution) needed to reproduce the experiments?
    \item[] Answer: \answerYes{} 
    \item[] Justification: Appendix H.
    
\item {\bf Code of ethics}
    \item[] Question: Does the research conducted in the paper conform, in every respect, with the NeurIPS Code of Ethics \url{https://neurips.cc/public/EthicsGuidelines}?
    \item[] Answer: \answerYes{} 
    \item[] Justification: The paper involves no human subjects, releases no new datasets, and uses only publicly available pretrained models cited in Section 4.1.

\item {\bf Broader impacts}
    \item[] Question: Does the paper discuss both potential positive societal impacts and negative societal impacts of the work performed?
    \item[] Answer: \answerYes{} 
    \item[] Justification: Appendix G.
    
\item {\bf Safeguards}
    \item[] Question: Does the paper describe safeguards that have been put in place for responsible release of data or models that have a high risk for misuse (e.g., pre-trained language models, image generators, or scraped datasets)?
    \item[] Answer: \answerNA{} 
    \item[] Justification: The paper does not release new pretrained models, datasets, or model weights. SHS is a training-free inference rule that operates on top of existing publicly released models, and therefore does not introduce additional misuse risk beyond that already present in those models.

\item {\bf Licenses for existing assets}
    \item[] Question: Are the creators or original owners of assets (e.g., code, data, models), used in the paper, properly credited and are the license and terms of use explicitly mentioned and properly respected?
    \item[] Answer: \answerYes{} 
    \item[] Justification: All pretrained models and the GPT2-large evaluator are cited at first use in Section 4.1.

\item {\bf New assets}
    \item[] Question: Are new assets introduced in the paper well documented and is the documentation provided alongside the assets?
    \item[] Answer: \answerYes{} 
    \item[] Justification: The reference implementation of SHS is released with a README documenting usage and dependencies.

\item {\bf Crowdsourcing and research with human subjects}
    \item[] Question: For crowdsourcing experiments and research with human subjects, does the paper include the full text of instructions given to participants and screenshots, if applicable, as well as details about compensation (if any)? 
    \item[] Answer: \answerNA{} 
    \item[] Justification: The paper does not involve crowdsourcing or human subject research.

\item {\bf Institutional review board (IRB) approvals or equivalent for research with human subjects}
    \item[] Question: Does the paper describe potential risks incurred by study participants, whether such risks were disclosed to the subjects, and whether Institutional Review Board (IRB) approvals (or an equivalent approval/review based on the requirements of your country or institution) were obtained?
    \item[] Answer: \answerNA{} 
    \item[] Justification: The paper does not involve crowdsourcing or human subject research.

\item {\bf Declaration of LLM usage}
    \item[] Question: Does the paper describe the usage of LLMs if it is an important, original, or non-standard component of the core methods in this research? Note that if the LLM is used only for writing, editing, or formatting purposes and does \emph{not} impact the core methodology, scientific rigor, or originality of the research, declaration is not required.
    \item[] Answer: \answerNA{} 
    \item[] Justification: LLMs appear only as pretrained generative backbones and as a standard perplexity evaluator (GPT2-large), not as a component of the proposed method.

\end{enumerate}

\end{document}